\def\publicpreprint{}
\documentclass{article} 
\usepackage[T1]{fontenc}
\usepackage{iclr2026_conference,times}

\usepackage{amsmath,amsfonts,bm}

\newcommand{\captiona}{{\em (a)}}
\newcommand{\captionb}{{\em (b)}}
\newcommand{\captionc}{{\em (c)}}

\def\eqref#1{equation~\ref{#1}}

\def\1{\bm{1}}

\DeclareMathAlphabet{\mathsfit}{\encodingdefault}{\sfdefault}{m}{sl}
\SetMathAlphabet{\mathsfit}{bold}{\encodingdefault}{\sfdefault}{bx}{n}

\newcommand{\E}{\mathbb{E}}

\DeclareMathOperator*{\argmax}{arg\,max}

\usepackage{booktabs}       
\usepackage{longtable}
\usepackage{array}          
\usepackage{amsfonts}       
\usepackage{microtype}      
\usepackage{xcolor}         
\usepackage{mdframed}       
\usepackage{graphicx}
\usepackage{placeins}       
\usepackage{hyperref}
\usepackage{url}
\hypersetup{colorlinks=true,linkcolor=blue,citecolor=blue,urlcolor=blue,
  pdftitle={Reinforcing the World's Edge: A Continual Learning Problem in the Multi-Agent-World Boundary}}

\ifdefined\publicpreprint
  \preprintcopy
  \hypersetup{
    pdfauthor={Dane Malenfant}
  }
\fi

\usepackage{amsmath}   
\usepackage{amsthm}    
\usepackage{amssymb}
\usepackage{enumitem}
\usepackage[capitalize,noabbrev]{cleveref}

\theoremstyle{plain}
\newtheorem{theorem}{Theorem}[section]
\newtheorem{lemma}{Lemma}[section]
\newtheorem{proposition}{Proposition}[section]
\newtheorem{corollary}{Corollary}[section]
\theoremstyle{definition}
\newtheorem{definition}{Definition}[section]
\newtheorem{assumption}{Assumption}[section]
\theoremstyle{remark}
\newtheorem{remark}{Remark}[section]

\newcommand{\Ind}{\mathbf{1}}
\newcommand{\Subseq}{\preccurlyeq} 
\newcommand{\norm}[1]{\left\lVert #1 \right\rVert}
\newcommand{\proofstep}[1]{\par\smallskip\noindent\textbf{#1.}\enspace}
\definecolor{coreaccent}{RGB}{0,114,178}
\definecolor{coreshade}{RGB}{242,248,252}
\definecolor{stabilityaccent}{RGB}{213,94,0}
\definecolor{stabilityshade}{RGB}{253,247,242}
\definecolor{valueaccent}{RGB}{0,135,105}
\definecolor{valueshade}{RGB}{242,250,247}
\newmdenv[backgroundcolor=coreshade,linecolor=coreaccent,
  linewidth=1.4pt,leftline=true,rightline=false,topline=false,bottomline=false,
  innerleftmargin=8pt,innerrightmargin=8pt,innertopmargin=5pt,innerbottommargin=5pt,
  skipabove=5pt,skipbelow=5pt]{corekey}
\newmdenv[backgroundcolor=stabilityshade,linecolor=stabilityaccent,
  linewidth=1.4pt,leftline=true,rightline=false,topline=false,bottomline=false,
  innerleftmargin=8pt,innerrightmargin=8pt,innertopmargin=5pt,innerbottommargin=5pt,
  skipabove=5pt,skipbelow=5pt]{stabilitykey}
\newmdenv[backgroundcolor=valueshade,linecolor=valueaccent,
  linewidth=1.4pt,leftline=true,rightline=false,topline=false,bottomline=false,
  innerleftmargin=8pt,innerrightmargin=8pt,innertopmargin=5pt,innerbottommargin=5pt,
  skipabove=5pt,skipbelow=5pt]{valuekey}

\newcommand{\Sa}{\mathcal{S}}
\newcommand{\Aa}{\mathcal{A}}
\newcommand{\Gg}{\mathcal{G}}
\newcommand{\Mm}{\mathcal{M}}
\newcommand{\MarkovGame}{\mathfrak{G}}
\newcommand{\Sigmaa}{\Sigma}
\newcommand{\Core}{\mathrm{Core}}
\newcommand{\Corerob}{\Core_{\mathrm{rob}}}
\newcommand{\HH}{\mathcal{H}}
\newcommand{\TV}{\mathrm{TV}}
\newcommand{\suc}{\mathsf{S}}
\crefname{assumption}{Assumption}{Assumptions}
\Crefname{assumption}{Assumption}{Assumptions}

\title{Reinforcing the World's Edge:\\
A Continual Learning Problem in the\\
Multi-Agent-World Boundary}

\author{Dane Malenfant \\ School of Computer Science\\McGill University\\Mila - The Qu\'ebec AI Institute\\
  \texttt{dane.malenfant@mail.mcgill.ca} \\
}

\begin{document}

\maketitle

\begin{abstract}
In a stationary decentralized Markov game, learning peers generate an episode-indexed sequence of induced MDPs for any focal agent. The joint game remains stationary while the focal agent's rewards and dynamics drift, forming an agent-centric continual reinforcement-learning problem. If peer policies are fixed within an episode, marginalizing them preserves the focal trajectory law and expected return for every focal policy. The focal agent need not itself be learning: peer updates induce the continual process. Success-conditioned reusable structure may therefore degrade across episodes. An \emph{invariant core} represents such structure through maximal abstract patterns appearing in a high fraction of successful focal trajectories. Here, invariant means shared across successes under a fixed trajectory law. Survival certifies a fixed pattern's coverage rather than an unchanged maximal frontier. We connect this structure's lifetime to peer learning and policy reuse. An established sharp total-variation conditioning bound limits coverage loss to $\frac{\varepsilon}{p_0}$, where $\varepsilon$ is trajectory-law drift, $p_0$ is reference success mass, and drift is smaller than the initial compatible-success mass. Combining this bound with peer-policy movement gives a certified $\Omega(\frac{1}{\eta})$ survival horizon for a candidate with positive coverage margin. Under an explicit effective-conflict condition---satisfied by exact policy gradient in an analytic class---the first-exit time is $\Theta(\frac{1}{\eta})$. With calibrated success mass and executability, the same certificate yields policy-value, library-selection, and transfer-regret guarantees.
An exactly solvable corridor confirms the structural predictions, including the inverse-rate lifetime ($R^2>0.9999$). Two registered 64-stream studies in continual control and cue-MNIST show that core erosion predicts impending failure and enables near-oracle intervention; an exploratory reanalysis of eight learned-partner Level-Based Foraging development pairings suggests the same erosion--failure link under peer learning.
\end{abstract}

\section{Introduction}
Standard stationary-MDP formulations model an agent interacting with a fixed world \citep{SuttonBarto2018}. In ordinary Dec-MARL, the joint Markov game may remain stationary while the transitions and rewards experienced by any one agent drift as its partners update. This peer-induced non-stationarity is a central difficulty of multi-agent learning \citep{Littman1994,claus1998dynamics,bowling2002multiagent}, often addressed through replay correction, opponent modeling, or recursive reasoning.

Consider one \emph{focal} agent in a stationary decentralized Markov game. During episode $e$, use the standard best-response lens \citep{ShohamLeytonBrown2009,Busoniu2008} to marginalize the peers' fixed policies into the transition and reward functions. The resulting single-agent MDP exactly matches the focal trajectory law and expected return (\Cref{prop:reduction}). As the peers update between episodes, these induced MDPs form the focal agent's continual reinforcement-learning problem \citep{khetarpal2020towards}. Peer-policy movement becomes an environment-variation budget, which scales with the learning rate under bounded smooth updates. The object of study is this agent-centric continual problem: the multi-agent system supplies the endogenous change. In the language of the agent--world boundary \citep{abel2025agency}, decentralized peer learning makes the boundary move.

Which reusable structure persists under peer updates? Conditioning on success is consequential: when the reference success mass is $p_0$, unconditional trajectory drift can be amplified by a factor of $\frac{1}{p_0}$. This sharp conditioning relation is established in imprecise-probability theory \citep[Proposition~3.2]{pelessoni2022dilation}. We combine it with peer-learning dynamics to derive finite-time structural survival, first-exit, value, and transfer guarantees.

\begin{figure}[t]
\centering
\includegraphics[width=0.97\linewidth]{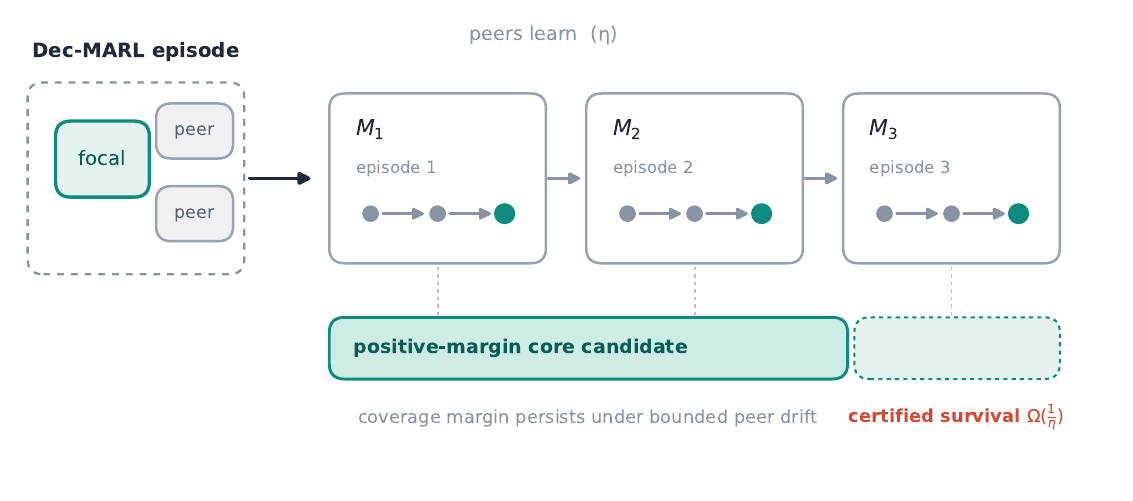}
\caption{\textbf{Structural lifetime under peer learning.} The joint Markov game remains stationary, while peer updates generate an episode-indexed sequence of induced MDPs for the focal agent. The sharp success-conditioning bound then certifies an $\Omega(\frac{1}{\eta})$ survival horizon for a fixed core candidate with positive coverage margin.}
\label{fig:reduction}
\end{figure}
\FloatBarrier

The \emph{invariant core} tracks success-conditioned reusable structure through maximal abstract patterns appearing in a high fraction of successful focal trajectories. \Cref{fig:reduction} summarizes the chain from peer-policy drift to induced-MDP variation, sharp conditional coverage loss, structural lifetime, and policy value. \Cref{thm:stab} gives a certified $\Omega(\frac{1}{\eta})$ lifetime; under the effective-conflict condition of \Cref{ass:progress}, realized analytically in \Cref{prop:p2example}, the first-exit time is $\Theta(\frac{1}{\eta})$. The empirical tests measure the chain's predictive and intervention consequences.

Combining the focal induced-MDP equivalence (\Cref{prop:reduction}) with the stability theorem (\Cref{thm:stab}) yields the central coverage bound. While its positive-margin premise holds, the success-conditioned coverage of a fixed candidate $u$ under a fixed focal probe obeys
\[
\begin{gathered}
c_{\mu_E}(u)\ge c_{\mu_0}(u)-\frac{\varepsilon_E^{\mathrm{peer}}}{p_0},
\qquad \varepsilon_E^{\mathrm{peer}}<p_0c_{\mu_0}(u),\\
\varepsilon_E^{\mathrm{peer}}
:=(H-1)\sum_{e=1}^{E}\sup_s
\TV\!\left(\pi_2^e(\cdot\mid s),\pi_2^{e-1}(\cdot\mid s)\right).
\end{gathered}
\tag{$\star$}
\]
Here $\pi_2^e$ is the one-peer specialization; \Cref{prop:reduction} gives the joint-peer version. Inequality ($\star$) combines established conditional-probability control with cumulative peer-policy movement (\Cref{tab:novelty}). Its structural consequence is a lifetime certificate: if each peer update moves by at most $c\eta$, a candidate with target slack $\bar\delta$ and margin $m$ survives for at least
\[
E_{\mathrm{exit}}^{\bar\delta}(u)
\;\ge\;
\left\lfloor\frac{mp_0}{c\eta(H-1)}\right\rfloor+1
\]
episodes (\Cref{cor:rate}). \Cref{thm:twosided} supplies the matching conditional order, while \Cref{thm:value,cor:value-survival} convert survival into value and transfer guarantees.

\paragraph{Results.}
\begin{enumerate}[topsep=2pt,itemsep=1pt,leftmargin=1.4em]
\item \textbf{Structural lifetimes under peer learning.} The coverage analysis separates positive-margin stability from exact-core discontinuity and yields unconditional $\Omega(\frac{1}{\eta})$ survival and a non-vacuous conditional $\Theta(\frac{1}{\eta})$ first-exit law (\Cref{prop:impossible,thm:stab,cor:rate,thm:twosided}).
\item \textbf{Performance consequences.} Calibrated success mass and executability convert surviving structure into policy-value, library-selection, and transfer-regret guarantees (\Cref{sec:value}).
\item \textbf{Downstream consequence and validation.} Exact population studies verify degradation, lifetime, and sharpness; two registered 64-stream confirmations establish predictive and near-oracle intervention value; an exploratory learned-peer LBF reanalysis links core erosion to old-partner failure (\Cref{sec:exp,app:exp}).
\end{enumerate}
\FloatBarrier

\section{Related Work}
\begin{table}[!h]
\centering
\caption{\textbf{From established tools to structural lifetimes.} The analysis combines conditional-probability control with peer-induced variation to quantify survival and performance consequences.}
\label{tab:novelty}
\small
\setlength{\tabcolsep}{3pt}
\renewcommand{\arraystretch}{1.08}
\begin{tabular}{@{}>{\raggedright\arraybackslash}p{0.18\linewidth} >{\raggedright\arraybackslash}p{0.32\linewidth} >{\raggedright\arraybackslash}p{0.42\linewidth}@{}}
\toprule
Literature & Established object & Role in the structural-lifetime analysis \\
\midrule
Planning landmarks & Necessary facts and orderings within a fixed planning task \citep{hoffmann2004ordered,richter2010lama}. & Quantitative coverage of abstract structure among stochastic successes under a changing trajectory law. \\
Frequent sequence mining & Discovery of constrained subsequences above an empirical support threshold \citep{girgin2005option,beedkar2019desq}. & A transition- and peer-drift bound on conditional support loss, yielding a certified survival time. \\
Total-variation conditioning & Sharp conditional-probability envelopes in imprecise-probability models \citep{pelessoni2022dilation}. & The established factor $\frac{1}{p_0}$ converts trajectory drift into coverage loss; accumulated peer drift then yields a structural lifetime. \\
Non-stationary MDPs & Dynamic-regret guarantees from reward and transition variation budgets \citep{EvenDar2009,pmlr-v119-cheung20a,pmlr-v139-mao21b}. & Variation controls the survival of reusable structure and, with calibration, its policy-value and transfer consequences. \\
MARL non-stationarity & Focal-agent or best-response induced MDPs, plus algorithms compensating for co-learning and replay shift \citep{Littman1994,ShohamLeytonBrown2009,Busoniu2008,foerster2017stabilising,Foerster2018LOLA}. & A peer-update budget combined with sharp conditional coverage yields an explicit structural lifetime and downstream decision guarantees. \\
Continual RL & Formalizations and taxonomies of ongoing adaptation under non-stationarity \citep{khetarpal2020towards,Elelimy2023ContinualRL}. & A success-conditioned structural witness whose coverage, lifetime, and value consequences are controlled under the continual sequence. \\
\bottomrule
\end{tabular}
\end{table}

The conditional inequality in \Cref{lem:condcov} specializes the total-variation model of \citet[Section~2.1 and Proposition~3.2]{pelessoni2022dilation}; its short proof is retained for self-containment. The contribution developed here is the structural-lifetime analysis: cumulative peer updates control coverage erosion, positive margins yield survival horizons, effective conflict gives a matching first-exit order, and calibrated execution connects survival to policy value and transfer. The focal induced-MDP representation identifies the source of variation, while established conditioning control supplies the probabilistic link in this chain.

Planning landmarks, RL bottlenecks, probabilistic landmarks, and sequence search supply useful candidate structures \citep{mcgovern2001automatic,simsek2009skill,chan2025landmark}. State abstraction \citep{10.5555/1867406.1867423,li2006towards,pmlr-v48-abel16} shrinks the alphabet, while options \citep{SuttonPrecupSingh1999,konidaris2009skill} supply executors. Our stability analysis transports these candidates across episodes; executability remains explicit in \Cref{sec:exec}.

Work on co-adapting agents addresses non-stationarity through replay correction, opponent modeling, and recursive reasoning \citep{hernandezleal2017survey,papoudakis2019dealing,He2016Opponent,Raileanu2018Modeling}. Zero-shot coordination and ad hoc teamwork seek partner-robust behavior \citep{hu2020other,lupu2021trajectory,stone2010adhoc}; the robust core (\Cref{def:robust}) identifies structure shared across success-enabling partners. Inequality ($\star$) adds a finite-time survival guarantee for that structure under peer learning.
\FloatBarrier

\section{The Invariant Core in Stationary Single-Agent MDPs}\label{sec:core}
The invariant core is defined first in a stationary MDP; \Cref{sec:marl} then introduces learning peers and the resulting environment drift. The two-route corridor of \Cref{fig:env} provides a running instance whose shared start-to-goal structure is captured by the core.

Let $\Mm=(\Sa,\Aa,P,R,H,\Gg)$ be a finite-horizon goal-conditioned MDP with finite $\Sa,\Aa$, start state $s_0\notin\Gg$, and successful trajectories $\suc$ terminating on first entry into $\Gg$. Here $H$ is the maximum number of labels, including the terminal label, so there are at most $H-1$ stochastic transitions. Write $u\Subseq v$ when string $u$ is a not-necessarily-contiguous subsequence of $v$.

An optional abstraction $\phi:\Sa\times\Aa\to\Sigmaa$ maps each decision pair to an abstract symbol \citep{li2006towards}. For the raw trajectory above, the induced string contains $\phi(s_t,a_t)$ for $0\le t<T$ and, on success, appends a distinguished terminal symbol $g_\Sigma$ for the transition into $\Gg$. Hence every string lies in $\Sigmaa^{\le H}$ (the empty string included), and every $\tau\in\suc$ satisfies $g_\Sigma\Subseq\phi(\tau)$. The identity abstraction recovers the raw decision-pair alphabet $\Sa\times\Aa$ together with $g_\Sigma$.

\begin{corekey}
\begin{definition}[Coverage and the $\delta$-core]\label{def:core}
Let $\mu$ be a reference distribution over trajectories (e.g.\ induced by a fixed \emph{probe} policy in $\Mm$), with $\mu(\suc)>0$. For $u\in\Sigmaa^{\le H}$ define its \emph{coverage}
\[
c_\mu(u)\;=\;\Pr_{\tau\sim\mu}\big[\,u\Subseq\phi(\tau)\,\mid\,\tau\in\suc\,\big].
\]
For $\delta\in[0,1)$ the \emph{high-coverage family} and the \emph{$\delta$-core} are
\[
\HH^\delta_\phi(\mu)\;=\;\big\{\,u\in\Sigmaa^{\le H}\;:\;c_\mu(u)\ge 1-\delta\,\big\},
\qquad
\Core^\delta_\phi(\mu)\;=\;\max_{\Subseq}\HH^\delta_\phi(\mu),
\]
the $\Subseq$-maximal elements of the family. The \emph{support core} is $\delta{=}0$. When $\mu$ has full support on $\suc$, $\HH^0_\phi(\mu)$ is the universal family $\{u:\forall\tau\in\suc,\,u\Subseq\phi(\tau)\}$ of subsequences shared by all successes, and $\Core^0_\phi(\mu)$ is its maximal frontier.
\end{definition}
\end{corekey}

In words: coverage $c_\mu(u)$ is the fraction of \emph{successful} trajectories in which the pattern $u$ appears; the family $\HH^\delta_\phi(\mu)$ collects the patterns present in all but a $\delta$-fraction of successes; and the core is the shortlist of maximal such patterns, from which every family member is recovered as a subsequence.

\begin{lemma}[The core generates its family]\label{lem:generate}
$\HH^\delta_\phi(\mu)$ is hereditary: if $u\Subseq w$ and $w\in\HH^\delta_\phi(\mu)$ then $u\in\HH^\delta_\phi(\mu)$. Consequently $\Core^\delta_\phi(\mu)$ is its finite generating frontier,
\[
\HH^\delta_\phi(\mu)\;=\;\big\{\,u:\exists\,v\in\Core^\delta_\phi(\mu),\ u\Subseq v\,\big\}.
\]
\end{lemma}

Thus membership means $u\in\HH^\delta_\phi(\mu)$, equivalently $u\Subseq v$ for some frontier element. Proofs are in \Cref{app:proofs,app:monitoring}.

\subsection{Existence and non-triviality}
For the graph characterization below, call a successful trajectory \emph{feasible} when every displayed action has positive transition probability along it.  The probe has \emph{full conditional support} when every feasible success has positive $\mu$-mass.  Define the \emph{action-labelled time-unrolled success graph} $G_{\suc}$ as follows: its vertices are $(s,t)$, and it has an edge
$(s,t)\xrightarrow{a}(s',t+1)$ whenever $P(s'\mid s,a)>0$ and that step lies on a $\mu$-positive success; the edge carries label $\phi(s,a)$.  Because time and action are retained, every $s_0$--$\Gg$ path in $G_{\suc}$ is a feasible successful trajectory, and full conditional support gives it positive $\mu$-mass.  Conversely, every $\mu$-positive success traces such a path.  A set of labelled edges is a \emph{step-cut} when it intersects every one of these paths.

\begin{theorem}[Existence and non-triviality]\label{thm:exist}
Fix $\mu$ with $\mu(\suc)>0$.
\begin{enumerate}[topsep=2pt,itemsep=1pt,label=(\alph*),leftmargin=1.6em]
\item \textbf{(Existence)} For every $\delta\in[0,1)$, some frontier element of $\Core^\delta_\phi(\mu)$ contains $g_\Sigma$; in particular $\Core^\delta_\phi(\mu)\ne\varnothing$.
\item \textbf{(Non-triviality)} A non-terminal symbol $\sigma$ with $c_\mu(\sigma)=1$ is a \emph{$\phi$-bottleneck}; some frontier element contains $(\sigma,g_\Sigma)$ as a subsequence and therefore strictly extends $(g_\Sigma)$. Under full conditional support, $\sigma$ is a bottleneck iff its labelled steps cut every start--goal path in the time-unrolled success graph.
\item A list of non-terminal bottleneck occurrences $\sigma_1,\dots,\sigma_m$ is \emph{order-consistent} if every $\mu$-positive success $\tau$ has positions $i_1<\dots<i_m$ with $\phi(\tau)_{i_j}=\sigma_j$. Symbols may repeat, but their occurrence positions must be distinct. Some element of $\Core^0_\phi(\mu)$ then contains $(\sigma_1,\dots,\sigma_m,g_\Sigma)$ and has length at least $m+1$.
\end{enumerate}
\end{theorem}

The goal symbol makes existence immediate; the content is the step-cut criterion. In key--door tasks, for example, ``find key'', ``reach door'', and ``open door'' can form ordered bottlenecks \citep{ChevalierBoisvert2018Minigrid,Hung2019TVT,Sun2023Contrastive}.

\subsection{Sound pruning and execution requirements}\label{sec:exec}
\begin{proposition}[Soundness (admissible pruning)]\label{prop:sound}
Suppose $\mu$ has full conditional support on the optimal successes (every optimal $\tau^\star$ has $\mu(\tau^\star)>0$). For any $u\in\HH^0_\phi(\mu)$ and any $\tau\in\suc$ with $\mu(\tau)>0$, $u\Subseq\phi(\tau)$; in particular every optimal successful trajectory contains $u$, so restricting search to $\{\tau:\forall u\in\Core^0_\phi(\mu),\,u\Subseq\phi(\tau)\}$ preserves \emph{every} optimal successful trajectory (in particular an optimum). Without the support hypothesis the guarantee is probe-relative: pruning preserves an optimum only among $\mu$-positive successes.
\end{proposition}
This proposition supplies admissible trajectory-level pruning. Expected-value policy optimization additionally requires an executor and an ordering of the frontier:

\begin{proposition}[Non-orderability]\label{prop:order}
Let a branching finite-horizon MDP have exactly two positive-mass successful abstract strings,
$(x,y,g_\Sigma)$ and $(y,x,g_\Sigma)$.  Their maximal common subsequences are $(x,g_\Sigma)$ and $(y,g_\Sigma)$, so
$\Core^0_\phi(\mu)=\{(x,g_\Sigma),(y,g_\Sigma)\}$ while neither $(x,y)$ nor $(y,x)$ is generated. Hence no concatenation of core symbols is forced to be a valid plan, and any executor must supply an order the core does not determine.
\end{proposition}

Model each symbol as an option $o_\sigma=(I_\sigma,\pi_\sigma,\beta_\sigma)$ on time-augmented states $(s,t)$ \citep{SuttonPrecupSingh1999}; its initiation set, policy $\pi_\sigma(\cdot\mid s,t)$, and termination rule may depend on the remaining horizon.
\begin{assumption}[Execution]\label{ass:exec}\leavevmode
(A1) the first option can start at $s_0$; (A2) each option emits its symbol on termination; (A3) consecutive options are chainable; (A4) the chain terminates almost surely within $H-1$ transitions; and (A5) its final termination lies in $\Gg$.
\end{assumption}

\begin{proposition}[Ordered-core Bellman decomposition]\label{prop:chain}
Let $(\sigma_1,\ldots,\sigma_r,g_\Sigma)$ be one ordered frontier string. Under \Cref{ass:exec}, its option chain reaches $\Gg$. For additive costs, define $V^\star(s,t)$ as the infimum expected remaining cost among policies that reach $\Gg$ almost surely by time $H-1$.

Suppose each $\sigma_i$ labels a step entering one fixed state $g_i$, and every success enters $g_1,\ldots,g_r$ in this order. Set $g_0=s_0$, and let the final $g_\Sigma$-labelled step enter a fixed goal state $g_{r+1}\in\Gg$. For each $i=1,\ldots,r+1$, let $\mathcal P_i(t)$ contain the segment policies that start from $(g_{i-1},t)$ and terminate at $g_i$ almost surely within the remaining horizon. Let $(C_i,T_i)$ denote a segment option's cost and terminal time. Suppose each supplied segment option belongs to $\mathcal P_i(t)$ and attains
\[
V^\star(g_{i-1},t)=\inf_{\pi_i\in\mathcal P_i(t)}\E[C_i+V^\star(g_i,T_i)\mid g_{i-1},t]
\qquad(i=1,\ldots,r+1),
\]
where $V^\star(g_{r+1},\cdot)=0$. Then the supplied option chain attains the goal-reaching comparator $V^\star$. The entered-state landmarks separate the time-unrolled graph, which makes this recursion Bellman-optimal.
\end{proposition}
\Cref{prop:chain} gives a Bellman decomposition for supplied options attaining the segment infima. Entered-state landmarks provide the separators needed for optimal chaining, while the abstraction and option library jointly determine executability.

\subsection{Definitional invariance vs.\ estimated drift}\label{sec:single-stab}
In a stationary MDP the true support core over all successes is policy-invariant. A finite empirical core can drift as sampling coverage improves; \Cref{thm:estimation} controls fixed-probe i.i.d.\ evaluation batches, not pooled rollouts from an adapting policy. In Dec-MARL, by contrast, the true core itself can move.

\section{Complexity of Core Computation}\label{sec:complexity}
For the support core, computing \emph{one longest} frontier element is a longest-common-subsequence problem. For $k$ successful strings, exact product-space dynamic programming visits
\[
O\!\left(\prod_{i=1}^{k}(H_i+1)\right)=O\!\left(H^{k}\right),
\]
and generalized LCS is NP-hard \citep{Maier1978}. The product-space dynamic program returns a longest frontier witness exactly. A pointwise abstraction shrinks the alphabet, while temporal abstraction shortens the strings. For $\delta>0$, support-constrained sequence-search procedures such as DESQ provide candidate-discovery machinery \citep{beedkar2019desq}. Given a discovered candidate set $\mathcal U$, a two-pointer subsequence scan evaluates all empirical coverages in $O(N|\mathcal U|H)$ time; the implemented length-1 and length-2 monitors instantiate this evaluator.

\subsection{Statistical cost: fixed-candidate membership testing}\label{sec:estimation}
For $N$ i.i.d.\ probe rollouts, let $N_{\suc}$ be the number of successes and
\[
\widehat c_N(u)=\frac{1}{N_{\suc}}\sum_{i:\tau_i\in\suc}\Ind[u\Subseq\phi(\tau_i)]\qquad(N_{\suc}>0).
\]
Set $\widehat c_N(u)=0$ when $N_{\suc}=0$.

\begin{theorem}[Fixed-candidate membership testing]\label{thm:estimation}
Let $p_0=\mu(\suc)>0$, $t\in(0,1]$, $\gamma>0$, and $\rho\in(0,1)$.
\begin{enumerate}[topsep=2pt,itemsep=1pt,label=(\alph*),leftmargin=1.6em]
\item If $N\ge \frac{1}{p_0t^2}\log\frac{4}{\rho}+\frac{8}{p_0}\log\frac{2}{\rho}$, then with probability at least $1-\rho$, $|\widehat c_N(u)-c_\mu(u)|\le t$.
\item For $M$ pre-specified candidates each $\gamma$-separated from $1-\delta$, $N=O(\frac{\log(\frac{M}{\rho})}{p_0\gamma^2})$ recovers every membership with probability at least $1-\rho$.
\item At a fixed interior threshold, one candidate and constant error require $N=\Omega(\frac{1}{p_0\gamma^2})$; the constant is threshold-dependent near $0$ and $1$.
\end{enumerate}
\end{theorem}
Part (a) combines two requirements. The second term ensures that the sample contains enough successful trajectories; conditional on those successes, the first term controls the coverage estimate. Parts (b) and (c) give the corresponding uniform upper rate and matching lower rate.

\begin{remark}[Estimation and drift are interchangeable budgets]\label{rem:est-drift}
With known $p_0$, sampling slack $t$ and drift slack $\frac{\varepsilon_e}{p_0}$ consume the same margin: if $\widehat c_N(u)-(1-\bar\delta)\ge t+\frac{\varepsilon_e}{p_0}$, then \Cref{thm:estimation,thm:stab} certify $u\in\HH^{\bar\delta}_\phi(\mu_e)$ with probability at least $1-\rho$. For unknown success mass, use the lower confidence bound in \Cref{prop:empirical-certificate}; an empirical success fraction alone need not be conservative.
\end{remark}

\section{Focal Induced MDPs and Peer-Drift Budgets}\label{sec:marl}
Let $\MarkovGame=(\Sa,\Aa_1,\ldots,\Aa_n,P,R_1,H,\Gg)$ be a stationary decentralized Markov game with focal agent $1$ \citep{Littman1994}. In episode $e$, let the other agents follow a joint state-based Markov kernel $\pi_{-1}^e(a_{-1}\mid s)$ that is fixed within the episode. Marginalizing them gives the induced MDP $\Mm_e=(\Sa,\Aa_1,P_e,R_e,H,\Gg)$:
\[
\begin{aligned}
P_e(s'\mid s,a_1)
&=\sum_{a_{-1}}P(s'\mid s,a_1,a_{-1})\,\pi_{-1}^e(a_{-1}\mid s),\\
R_e(s,a_1)
&=\sum_{a_{-1}}R_1(s,a_1,a_{-1})\,\pi_{-1}^e(a_{-1}\mid s).
\end{aligned}
\]

\begin{proposition}[Focal induced-MDP equivalence]\label{prop:reduction}
For every episode $e$ and every focal policy based on its state--action history, the marginal focal state--action trajectory law in $\MarkovGame$ equals its law in $\Mm_e$, and expected additive return agrees. Hence every fixed success event, abstraction, coverage, and core transports exactly to $\Mm_e$. Put $D_e:=\sup_s\TV(\pi_{-1}^e(\cdot\mid s),\pi_{-1}^{e-1}(\cdot\mid s))$ and $R_{\max}:=\max_{s,a_1,a_{-1}}|R_1(s,a_1,a_{-1})|$. Then
\[
\norm{P_e-P_{e-1}}_{1,\infty}\le 2D_e,\qquad
\sup_{s,a_1}|R_e(s,a_1)-R_{e-1}(s,a_1)|\le 2R_{\max}D_e.
\]
If decentralized peers act conditionally independently, then
$D_e\le\sum_{j=2}^n\sup_s\TV(\pi_j^e(\cdot\mid s),\pi_j^{e-1}(\cdot\mid s))$.
\end{proposition}
Thus $\mathcal C=(\Mm_0,\Mm_1,\ldots)$ is the agent-centric continual RL process faced by the focal agent: the multi-agent task and joint game stay fixed, while peer learning generates drift in the focal agent's induced rewards and dynamics. Attributing this drift to peers requires a fixed focal probe. The notation henceforth specializes to one peer ($n=2$), with policy $\pi_2^e$ and fixed goal-reaching event and probe law $\suc,\mu_e$ in $\Mm_e$. Persistent peer memory, shared randomness, or within-episode learning must be included in an augmented state; otherwise they need not induce a Markov kernel on $s$.

\paragraph{Scope of the reduction and lifetime bound.}
Every comparison fixes a common trajectory space, initial law, focal probe, success event $\suc$, abstraction $\phi$, and candidate. The kernel-based results assume a fully observed Markov state and peers fixed within each episode. Augmenting latent peer memory gives a latent-state model, not necessarily an observable focal MDP. Applying these results to a history or belief state requires a common sufficient representation and a separately established kernel-drift bound on that representation. The law-level \Cref{lem:condcov} applies directly whenever a common-space trajectory-TV bound is available. Reward-only drift enters the value analysis through $G_{e,k}$ and its calibration radius.

\paragraph{Dependency chain.}
For a fixed candidate, the argument follows the chain
\[
\begin{aligned}
\text{peer movement}
&\Longrightarrow \text{induced-kernel drift}
\Longrightarrow \text{trajectory-law drift}\\
&\Longrightarrow \text{conditional coverage loss}
\Longrightarrow \text{survival time}
\Longrightarrow \text{policy value}.
\end{aligned}
\]
\Cref{prop:reduction} establishes the first link, \Cref{lem:sim} the second, and \Cref{lem:condcov} the third. Positive margin yields survival through \Cref{thm:stab,cor:rate}; effective peer progress yields the two-sided law in \Cref{thm:twosided}; and success mass plus performance calibration yield value in \Cref{thm:value}. Each conclusion retains the premises of the preceding link.

\subsection{Impossibility: the exact core is discontinuous}
\begin{proposition}[No modulus of continuity for exact membership]\label{prop:impossible}
There is a family of games and peer policies $\{\pi_2^\alpha\}_{\alpha\in[0,\frac{1}{2}]}$ with induced kernels $P_\alpha$ such that $\norm{P_\alpha-P_0}_{1,\infty}\le 2\alpha\to 0$, yet a prototype $\sigma^\star$ generated by the exact core ($\sigma^\star\in\HH^0_\phi(\mu_0)$) satisfies $\sigma^\star\notin\HH^0_\phi(\mu_\alpha)$ for \emph{every} $\alpha>0$. Hence the map $P\mapsto\Ind[c_{\mu_P}(u)=1]$ from the induced kernel to the exact-membership indicator admits no modulus of continuity, and the maximal frontier itself jumps.
\end{proposition}
Two routes sharing only start and goal witness the result: any positive probability on the bypass removes a top-route symbol from coverage one, although the induced kernel converges to the top-only kernel.

\subsection{A coverage margin makes $\delta$-core membership Lipschitz-stable}

\begin{lemma}[Trajectory simulation]\label{lem:sim}
For a fixed focal probe policy and common initial law, the trajectory laws under $\Mm_e$ and a reference $\Mm_0$ satisfy $\norm{\mu_e-\mu_0}_{\TV}\le \varepsilon_e:=\frac{1}{2}(H-1)\,\norm{P_e-P_0}_{1,\infty}$.
\end{lemma}
\begin{stabilitykey}
\begin{lemma}[Sharp success-conditioning step]\label{lem:condcov}
Let $\mu,\nu$ be trajectory laws, let $A\subseteq\suc$, and put $p=\mu(\suc)>0$. If $\TV(\mu,\nu)\le\varepsilon<\mu(A)$, then $\nu(\suc)>0$ and
\[
\nu(A\mid\suc)\ge \mu(A\mid\suc)-\frac{\varepsilon}{p}.
\]
The coefficient $\frac{1}{p}$ is best possible.
\end{lemma}
\begin{theorem}[$\delta$-family stability]\label{thm:stab}
Let $p_0=\mu_0(\suc)>0$, $\varepsilon_e=\frac{1}{2}(H-1)\,\norm{P_e-P_0}_{1,\infty}$, and suppose $\varepsilon_e<(1-\delta)p_0$. If $u\in\HH^\delta_\phi(\mu_0)$ (i.e.\ $u$ is generated by the $\delta$-core of $\mu_0$) then $u\in\HH^{\delta'}_\phi(\mu_e)$ with
\[
\delta'\;\le\;\delta+\frac{\varepsilon_e}{p_0}\;<\;1 .
\]
Equivalently, $\HH^\delta_\phi(\mu_0)\subseteq\HH^{\delta+\frac{\varepsilon_e}{p_0}}_\phi(\mu_e)$. The trajectory-law constant is worst-case tight.
\end{theorem}
\end{stabilitykey}
The proof follows \Cref{lem:sim}~$\Longrightarrow$~\Cref{lem:condcov}~$\Longrightarrow$~\Cref{thm:stab}. The middle lemma specializes established total-variation conditioning \citep[Proposition~3.2]{pelessoni2022dilation}; \Cref{app:proofs} gives the correspondence and a self-contained proof. The theorem controls coverage; membership in a fixed target family requires positive initial margin, and the maximal frontier may change.

Fix reference slack $\delta_{\mathrm{ref}}\in[0,1)$ and target slack $\bar\delta\in[0,1)$.  Define the surviving reference family and the target-family exit time by
\[
\mathcal K_e^{\delta_{\mathrm{ref}}\to\bar\delta}
:=\HH^{\delta_{\mathrm{ref}}}_\phi(\mu_0)\cap\HH^{\bar\delta}_\phi(\mu_e),
\qquad
E_{\mathrm{exit}}^{\bar\delta}(u)
:=\inf\{e\ge0:u\notin\HH^{\bar\delta}_\phi(\mu_e)\},
\]
with $\inf\varnothing=\infty$.
\begin{proposition}[Ordered erosion: the surviving structure is again a core]\label{prop:erosion}
At each episode $\mathcal K_e^{\delta_{\mathrm{ref}}\to\bar\delta}$ is hereditary and therefore has a finite maximal frontier.  If $u\Subseq v$ and both lie in the reference family, then
$E_{\mathrm{exit}}^{\bar\delta}(u)\ge E_{\mathrm{exit}}^{\bar\delta}(v)$. Re-entry can occur under non-monotone peer updates.
\end{proposition}

\subsection{From peer learning rate to boundary drift}
The inequalities $\norm{P_e-P_0}_{1,\infty}\le\sum_{j\le e}\norm{P_j-P_{j-1}}_{1,\infty}=:V_e$ (a variation budget as in \citealp{EvenDar2009,pmlr-v119-cheung20a,pmlr-v139-mao21b}) and $\norm{P_e-P_{e-1}}_{1,\infty}\le 2\sup_s\TV(\pi_2^e(\cdot\mid s),\pi_2^{e-1}(\cdot\mid s))$ link the peer's learning dynamics directly to induced-kernel drift.

\begin{corollary}[Certified survival horizon $\Omega(\frac{1}{\eta})$]\label{cor:rate}
Let $H\ge2$, $\delta_0:=1-c_{\mu_0}(u)<\bar\delta$, and $m:=\bar\delta-\delta_0>0$. If $\sup_s\TV(\pi_2^e,\pi_2^{e-1})\le c\eta$ for every $e\ge1$, then $V_E\le2c\eta E$ and $u\in\HH^{\bar\delta}_\phi(\mu_E)$ for every
\[
E\;\le\;E^\star=\Big\lfloor\frac{m\,p_0}{c\,\eta\,(H-1)}\Big\rfloor=\Omega\!\left(\frac{1}{\eta}\right),
\]
with $c,p_0,m,H$ fixed as $\eta$ varies.
\end{corollary}
\begin{remark}[One-sided guarantee]\label{rem:onesided}
Membership holds through $E^\star$, hence $E_{\mathrm{exit}}^{\bar\delta}(u)\ge E^\star+1$. A zero-gradient peer may preserve $u$ forever, so this unconditional result is $\Omega(\frac{1}{\eta})$, not $\Theta(\frac{1}{\eta})$.
\end{remark}
\begin{remark}[Score-to-TV control]\label{rem:pg}
Suppose $\pi_\theta$ is differentiable, $\sup_{s,a,\vartheta}\|\nabla_\theta\log\pi_\vartheta(a\mid s)\|_2\le B$ along the update segment, and $\|d_e\|_2\le G$.  Integrating the policy derivative along $\theta_e+t\eta d_e$ gives, for every state,
\[
\TV(\pi_{\theta_{e+1}},\pi_{\theta_e})
\le\frac{\eta}{2}\int_0^1\sum_a\pi_{\theta_e+t\eta d_e}(a\mid s)
\big|\langle\nabla\log\pi_{\theta_e+t\eta d_e}(a\mid s),d_e\rangle\big|\,dt
\le\frac{BG}{2}\eta.
\]
Thus the corollary holds with $c=\frac{BG}{2}$.  A softmax policy with feature norm at most $F$ has score norm at most $2F$, giving $c=FG$.
\end{remark}

\subsection{A two-sided lifetime law under peer progress}\label{sec:twosided}
\Cref{cor:rate} is one-sided because its premise $\sup_s\TV(\pi_2^e,\pi_2^{e-1})\le c\eta$ is an \emph{upper} bound: a peer that stops moving satisfies it without eroding coverage. A matching upper bound on the lifetime therefore requires the peer to keep moving \emph{and} to move in a direction that actually erodes the pattern. Let $A=\{u\Subseq\phi(\tau)\}\cap\suc$ and write
\[
c(\theta)\;=\;\frac{\mu_\theta(A)}{\mu_\theta(\suc)}
\]
for the coverage of $u$ as a function of the peer parameter $\theta$; on a finite MDP with a smooth (e.g.\ softmax) parameterization, $c$ is smooth wherever $\mu_\theta(\suc)>0$.

\begin{assumption}[Effective peer progress]\label{ass:progress}
Write $E_{\mathrm{exit}}:=E_{\mathrm{exit}}^{\bar\delta}(u)$. For updates $\theta_{e+1}=\theta_e+\eta d_e$, assume before first exit that
\begin{enumerate}[topsep=2pt,itemsep=0pt,label=(P\arabic*),leftmargin=1.8em]
\item \textbf{Bounded step.} $\norm{d_e}_2\le G$;
\item \textbf{Effective conflict.} For every $1\le E<E_{\mathrm{exit}}$, $\frac1E\sum_{e<E}-\langle\nabla c(\theta_e),d_e\rangle\ge\lambda>0$;
\item \textbf{Smoothness.} $\nabla_\theta c$ is $L_c$-Lipschitz on the segment $[\theta_e,\theta_{e+1}]$.
\end{enumerate}
\end{assumption}

\begin{proposition}[Non-vacuity: exact policy gradient in a two-route game]\label{prop:p2example}
In a two-route game, a peer chooses \textsf{keep}, entering the state labelled $u$, or \textsf{divert}, entering a bypass; either route then reaches the goal. Parameterize
\[
\pi_\theta(\textsf{divert})=\varsigma(\theta),\qquad
c(\theta)=\pi_\theta(\textsf{keep})=\varsigma(-\theta),
\qquad \varsigma(x)=(1+e^{-x})^{-1}.
\]
Let the divert reward exceed the keep reward by $\Delta>0$ and use exact gradient ascent on expected reward. For target coverage $\ell:=1-\bar\delta$ with $0<\ell<c_0:=c(\theta_0)<1$, put
\[
r_\star:=\min_{x\in[\ell,c_0]}x(1-x)>0,\qquad m:=c_0-\ell.
\]
Then (P2) holds termwise with $\lambda=\Delta r_\star^2$; also $G=\frac{\Delta}{4}$, $L_c=\frac{1}{4}$, and $\TV(\pi_{\theta_{e+1}},\pi_{\theta_e})\le\frac{\Delta}{16}\eta$. For $\eta\le\frac{64r_\star^2}{\Delta}$,
\[
\left\lfloor\frac{8m}{\Delta\eta}\right\rfloor+1
\;\le\;E_{\mathrm{exit}}\;\le\;
\left\lfloor\frac{2m}{\Delta r_\star^2\eta}\right\rfloor+1,
\]
so the two-sided theorem is non-vacuous.
\end{proposition}

\begin{lemma}[Cumulative erosion]\label{lem:erode}
Under \Cref{ass:progress}, if $\eta\le\frac{\lambda}{L_cG^2}$ then for every integer $1\le E<E_{\mathrm{exit}}$
\[
c(\theta_{E})\;\le\;c(\theta_0)-\kappa\,\eta E,\qquad \kappa:=\frac{\lambda}{2} .
\]
\end{lemma}

\begin{theorem}[Two-sided lifetime law]\label{thm:twosided}
Fix a target family $\HH_\phi^{\bar\delta}$, let $\delta_0:=1-c_{\mu_0}(u)$ be the \emph{exact} initial slack of $u$, and let $m:=\bar\delta-\delta_0>0$. Under the premise of \Cref{cor:rate} together with \Cref{ass:progress} for the same target $\bar\delta$, and for $\eta\le\frac{\lambda}{L_cG^2}$, the first exit episode obeys
\[
\Big\lfloor\frac{m\,p_0}{c\,\eta\,(H-1)}\Big\rfloor+1
\;\;\le\;\;E_{\mathrm{exit}}\;\;\le\;\;
\Big\lfloor\frac{2m}{\lambda\,\eta}\Big\rfloor+1 .
\]
Hence $E_{\mathrm{exit}}=\Theta(\frac{1}{\eta})$ when all displayed constants are uniform in $\eta$.
\end{theorem}

\begin{remark}[Effective conflict]\label{rem:dichotomy}
(P2) requires persistent coverage-eroding progress; the bounded-step survival certificate also permits cycling or stalled peers. For stochastic updates, the two-sided conclusion holds on realized paths satisfying (P2); a mean-gradient condition alone is insufficient. \Cref{prop:p2example} verifies (P2) pointwise for exact policy gradient.
\end{remark}

\paragraph{From cumulative erosion to a windowed monitor.}
Writing $a_j=-\langle\nabla c(\theta_j),d_j\rangle$, smoothness and bounded steps give
\[
\left|c(\theta_{t-w})-c(\theta_t)-\eta\sum_{j=t-w}^{t-1}a_j\right|
\le\frac{L_cG^2}{2}\eta^2w .
\]
\Cref{prop:window-erosion} carries this bound through the positive part and confidence intervals. It yields a windowed erosion signal under progress on that window; predicting a future exit additionally requires continued progress.

\subsection{Membership drift vs.\ executability drift}
\begin{proposition}[Two drift modes]\label{prop:exec-drift}
A peer update can cause two distinct failures. Under \emph{membership drift}, $u$ leaves $\HH^0_\phi(\mu_{e+1})$. Under \emph{executability drift}, $u$ remains in the family, but its options violate chainability or bounded-horizon goal termination (A3)--(A5). Membership monitoring cannot detect the second failure: the ordered chain's success probability can become arbitrarily small while $\HH^0_\phi$ and its frontier remain unchanged.
\end{proposition}

\subsection{From core survival to policy value}\label{sec:value}
Membership alone cannot provide a value bound because execution can still fail (\Cref{prop:exec-drift}). Value control therefore also requires executability of the preserved pattern. Let $\{\pi^k\}_{k=1}^K$ be a supplied library of focal policies or option chains, with fixed candidates $u_k$. At episode $e$, let $\mu_{e,k}$ be the monitoring law and $\nu_{e,k}$ the deployment law for $\pi^k$ on a common padded trajectory space. Put
\[
A_k:=\suc\cap\{u_k\Subseq\phi(\tau)\},\qquad
\Gamma_{e,k}:=\mu_{e,k}(A_k)
=p_{e,k}c_{\mu_{e,k}}(u_k),
\]
where $p_{e,k}:=\mu_{e,k}(\suc)>0$. Thus $\Gamma_{e,k}$ is the \emph{core-compatible success mass}, not success-conditional coverage alone. Let $G_{e,k}(\tau)$ be total return and $V_{e,k}:=\E_{\nu_{e,k}}G_{e,k}$.

The value bridge has three explicit links.  First, structural survival supplies
$c_{\mu_{e,k}}(u_k)\ge1-\delta$.  Second, a success-mass floor supplies
$\Gamma_{e,k}=p_{e,k}c_{\mu_{e,k}}(u_k)\ge\underline p_k(1-\delta)$.  Third, performance calibration converts this compatible mass into deployment value.  The next assumption states exactly the quantities required for the third link.

\begin{assumption}[Performance calibration]\label{ass:calibration}
For each $(e,k)$ there are upper bounds $\chi_{e,k},\xi_{e,k},b_{e,k}\ge0$ and a scale $r_k\ge0$ such that
\[
\begin{gathered}
|\nu_{e,k}(A_k)-\mu_{e,k}(A_k)|\le\chi_{e,k},\qquad
\E_{\nu_{e,k}}\!\left[|G_{e,k}|\Ind[A_k^c]\right]\le b_{e,k},\\
|G_{e,k}-r_k|\Ind[A_k]\le\xi_{e,k}\Ind[A_k]\quad\nu_{e,k}\text{-a.s.}
\end{gathered}
\]
Write $\zeta_{e,k}:=\xi_{e,k}+b_{e,k}+r_k\chi_{e,k}$. The radius $\zeta_{e,k}$ records execution error, reward outside the compatible event, and shift in that event's probability.
\end{assumption}

The event-specific shift is sufficient; full trajectory TV also upper-bounds it. On-policy monitoring has $\chi_{e,k}=0$. Independent monitoring and deployment samples can bound this scalar event difference (\Cref{app:monitoring}).

\begin{valuekey}
\begin{theorem}[Core-compatible value and transfer]\label{thm:value}
Under \Cref{ass:calibration},
\[
\left|V_{e,k}-r_k\Gamma_{e,k}\right|\le\zeta_{e,k}.
\]
Hence if $u_k\in\HH^\delta_\phi(\mu_{e,k})$ and $p_{e,k}\ge\underline p_k>0$, then
\[
V_{e,k}\ge r_k\underline p_k(1-\delta)-\zeta_{e,k}.
\]
For selection, suppose the scales and valid calibration bounds are available before choosing a policy. For a finite episode set $\mathcal E$, let $\mathcal I:=\mathcal E\times\{1,\ldots,K\}$. Estimate $\Gamma_{e,k}$ from $N_{e,k}$ independent monitoring rollouts by the empirical mean $\widehat\Gamma_{e,k}$ of $\Ind[A_k]$, and set
\[
\beta_{e,k}:=\sqrt{\frac{\log(\frac{2|\mathcal I|}{\rho})}{2N_{e,k}}},
\qquad
L_{e,k}:=r_k\widehat\Gamma_{e,k}-r_k\beta_{e,k}-\zeta_{e,k}.
\]
With probability at least $1-\rho$, the selector $\widehat k_e\in\argmax_k L_{e,k}$ satisfies simultaneously for every $e\in\mathcal E$ and comparator $k$,
\[
V_{e,k}-V_{e,\widehat k_e}
\le 2\!\left(r_k\beta_{e,k}+\zeta_{e,k}\right).
\]
Consequently, taking $k_e^\star\in\argmax_k V_{e,k}$ and summing,
\[
\sum_{e\in\mathcal E}\!\left(V_{e,k_e^\star}-V_{e,\widehat k_e}\right)
\le 2\sum_{e\in\mathcal E}\!\left(r_{k_e^\star}\beta_{e,k_e^\star}+\zeta_{e,k_e^\star}\right).
\]
\end{theorem}
\end{valuekey}

\begin{corollary}[Value survives with the core]\label{cor:value-survival}
Suppose the monitoring law for $(\pi^k,u_k)$ is the fixed-probe law of \Cref{cor:rate}, $p_{e,k}\ge\underline p_k$ and $\zeta_{e,k}\le\bar\zeta_k$ through its certified horizon $E^\star$. Then, for every $e\le E^\star$,
\[
V_{e,k}\ge r_k\underline p_k(1-\bar\delta)-\bar\zeta_k .
\]
Within this horizon, the same drift bound supplies
$p_{e,k}\ge p_{0,k}-\varepsilon_{e,k}\ge p_{0,k}(1-m)>0$,
so one may take $\underline p_k=p_{0,k}(1-m)$. Execution and reward calibration then give an $\Omega(\frac{1}{\eta})$ policy-value floor.
\end{corollary}
A supplied policy or option library can be evaluated by monitoring core-compatible success mass together with its execution radius. \Cref{thm:value} converts these quantities into selection and transfer-regret guarantees, while \Cref{prop:exec-drift} requires monitoring both.

\subsection{The robust core}
\begin{definition}[Robust (partner-universal) core]\label{def:robust}
For success-enabling peer policies $\Pi_2^+$, define
$\HH_{\mathrm{rob}}=\bigcap_{\pi_2\in\Pi_2^+}\HH^0_\phi(\mu(\pi_2))$ and $\Corerob=\max_\Subseq\HH_{\mathrm{rob}}$. The intersection is taken over families, not frontiers.
\end{definition}
For any peer policy $\pi_2$, write
$\operatorname{Supp}_{\suc}(\pi_2):=\{\tau\in\suc:\mu(\pi_2)(\tau)>0\}$.
\begin{proposition}[Overlap collapses to the robust core]\label{prop:overlap}
The intersection of any two success-enabling episode families contains $\HH_{\mathrm{rob}}$. For a visited peer sequence $\{\pi_2^e\}_{e\ge0}\subseteq\Pi_2^+$, if
\[
\bigcup_{e\ge0}\operatorname{Supp}_{\suc}(\pi_2^e)
=\bigcup_{\pi_2\in\Pi_2^+}\operatorname{Supp}_{\suc}(\pi_2),
\]
then $\bigcap_{e\ge0}\HH^0_\phi(\mu_e)=\HH_{\mathrm{rob}}$. In SRC, top and bottom routes share only $(s_0,g)$, so $\Corerob=\{(s_0,g)\}$.
\end{proposition}

\section{Empirical Validation}\label{sec:exp}
\begin{figure}[t]
\centering
\includegraphics[width=0.98\linewidth]{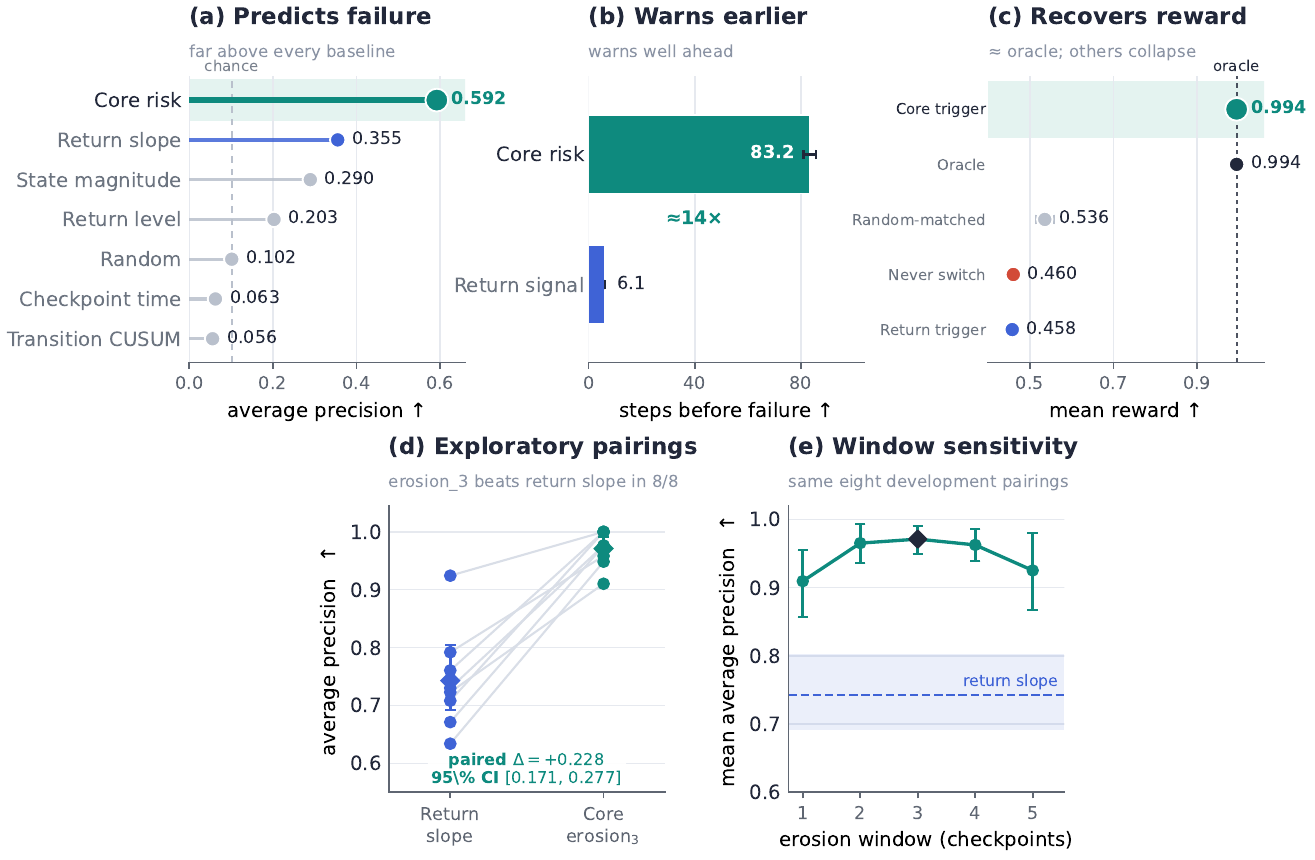}
\caption{\small\textbf{Core erosion and performance failure.} \emph{(a)--(c), E10 confirmation:} Across 64 streams, core risk predicts failure, warns earlier, and triggers near-oracle control. \emph{(d)--(e), E11 exploratory reanalysis:} Under learned LBF convention shift, core erosion beats return slope in all eight development pairings; panel (e) shows sensitivity across erosion windows. Bars, bands, and error bars are paired 95\% intervals resampling complete streams or pairings.}
\label{fig:main-experiments}
\end{figure}
The experiments serve three roles. E1--E9 test the structural laws in the two-player Switch--Rock Corridor (SRC), including discontinuity, stability, inverse-rate lifetime, sharpness, estimation, and executability. Exact population coverage gives first-exit time linear in $\frac{1}{\eta}$ ($R^2>0.9999$), and all 50 finite-rollout grid exits match their exact-grid counterparts. E10 and E12 prospectively test whether coverage erosion predicts failure and supports intervention in continual control and cue-MNIST. E11 explores the complete peer-update--coverage--failure chain in LBF. In these experiments, \emph{core erosion} means a decline in the coverage, or its lower confidence bound, of a fixed core candidate rather than a distance between two maximal frontiers.

Average precision (AP) summarizes the precision--recall ranking of prospective failure labels; higher AP means better separation of impending failures from eligible non-failures. Control and cue-MNIST are frozen-design confirmations: their protocols and analyses were fixed on development seeds before evaluation on disjoint confirmation seeds.

\paragraph{Formal correspondence of the measured objects.}
The four suites use the same conditional-event quantity at different levels.
\begin{enumerate}[topsep=2pt,itemsep=1pt,leftmargin=1.5em,label=(\roman*)]
\item \textbf{SRC (E1--E9):} $\mu_e$ is the fixed focal-probe trajectory law under peer checkpoint $e$, $\suc$ is goal reachability, and $A=\suc\cap\{u\Subseq\phi(\tau)\}$.  Thus the plotted population quantity is exactly $c_{\mu_e}(u)=\mu_e(A\mid\suc)$ from \Cref{def:core}.
\item \textbf{Continual control (E10):} draw $X_0$ uniformly from the fixed reference interval $X$ and evaluate the next state under the two extreme disturbances. The candidate event $A_{\mathrm{ctrl}}$ occurs when both next states remain in $X$. Because the entire reference population is designated successful, its conditional coverage is exactly the normalized size of the estimated viability set:
$\mu_t(A_{\mathrm{ctrl}}\mid\suc)=\frac{\operatorname{Leb}(\widehat K_t)}{\operatorname{Leb}(X)}=\widehat c_t$. Emitting one symbol on $A_{\mathrm{ctrl}}$ makes this a length-one coverage monitor with $p_t=1$.
\item \textbf{Cue-MNIST (E12):} draw index $I$ uniformly from the fixed probe set and let $\mu_t$ be the empirical law of
$Z_t=(y_I,f_t(x_I,q_{\mathrm{ref}}(y_I)),(f_t(x_I,c))_{c\in U_3})$, where $q_{\mathrm{ref}}(y)=(y+1)\bmod10$, $U_3=\{0,\ldots,9\}$, and $f_t(x,c)$ predicts from the image rendered with cue $c$. On this common coordinate space define the fixed events $\suc_{\mathrm{cue}}=\{Z_{\mathrm{ref}}=Z_y\}$ and
$A_{\mathrm{cue}}=\suc_{\mathrm{cue}}\cap\{Z_c=Z_y\ \forall c\in U_3\}$.  Then $\widehat c_t=\mu_t(A_{\mathrm{cue}}\mid\suc_{\mathrm{cue}})$ exactly; the reported $C_t$ is its Wilson lower bound.
\item \textbf{LBF (E11):} $\mu_t$ is again a fixed focal-probe trajectory law under a learned peer checkpoint, $\suc$ is successful coordination, $u$ is the discovered left-route sequence, and $L_t$ lower-bounds $c_{\mu_t}(u)$. This peer-induced instantiation uses exploratory erosion in place of the registered absolute-risk predictor; the three-checkpoint window matches the already-frozen failure horizon.
\end{enumerate}
Consequently E1--E9 test the structural laws directly, E10 and E12 test the downstream prediction-and-intervention use of the same coverage functional, and E11 traces the full peer-update--coverage--failure chain under learned partners.  Detailed constructions appear in \Cref{app:theorytests,app:continual-control,app:cue-mnist,app:lbf-exploratory}.

The registered studies compare specified monitoring signals and intervention triggers within each task. In continual control, 64 confirmation streams give core-risk AP $0.592$ versus $0.355$ for return slope. Core-triggered reward is $0.9944$ versus $0.4578$ for the return trigger, $0.4604$ without intervention, and $0.5356$ for random-matched switching, close to the $0.9945$ oracle (\Cref{fig:main-experiments}).

In continual cue-MNIST, 64 untouched supervised streams give erosion AP $0.384$ versus $0.345$ for cue association (paired $+0.039$, 95\% CI $[0.022,0.055]$). Core-triggered replay obtains $0.9500$ reward versus $0.9382$ for an accuracy trigger (paired $+0.0117$, $[0.0113,0.0121]$), matching the $0.9502$ oracle (\Cref{fig:mnist-confirmation}).

In a controlled Level-Based Foraging coordination fork, peers use a feed-forward logistic route policy updated by sampled REINFORCE. An exploratory reanalysis of eight development pairings suggests that fixed-probe, three-checkpoint core erosion predicts subsequent old-partner failure, reaching AP $0.971$ versus $0.743$ for return slope (paired difference $0.228$, 95\% CI $[0.171,0.277]$; \Cref{fig:main-experiments}). The registered absolute-risk score attained AP $0.526$, below return slope; the post-hoc switch to erosion used a window matching the already-frozen three-checkpoint failure horizon. Evaluation on the same development pairings remains exploratory. The coordination-fork schematic is shown at right in \Cref{fig:task-schematics}.

\FloatBarrier
\section{Conclusion}
Ordinary Dec-MARL presents each focal agent with a continual RL problem as its peers update, even though the joint Markov game remains stationary. Combining established sharp conditioning bounds with peer-induced trajectory drift gives a structural-lifetime analysis: positive coverage margins certify unconditional $\Omega(\frac{1}{\eta})$ survival, effective conflict yields a non-vacuous conditional $\Theta(\frac{1}{\eta})$ first-exit law, and calibrated executability gives policy-value and transfer guarantees. This chain connects peer learning to the persistence and usefulness of reusable structure. Registered E10 and E12 show that core erosion predicts failure and enables near-oracle intervention; exploratory E11 suggests that erosion also anticipates peer-induced old-partner failure.

\paragraph{Reproducibility statement.}
Complete proofs, protocols, per-seed statistics, and fixed-seed figure-regeneration code are supplied in \Cref{app:proofs,app:exp}.

\bibliographystyle{iclr2026_conference}
\bibliography{iclr2026_conference}

\appendix
\section{Notation}\label{app:notation}
\begin{table}[ht]
\centering\scriptsize
\renewcommand{\arraystretch}{1.15}
\begin{tabular}{@{}p{0.19\linewidth}@{\hspace{6pt}}p{0.74\linewidth}@{}}
\toprule
\multicolumn{2}{@{}l}{\textbf{Task and trajectories}}\\
\midrule
$\Mm=(\Sa,\Aa,P,R,H,\Gg)$ & finite-horizon goal-conditioned MDP: states, actions, kernel, reward, horizon, goal set; fixed start $s_0$ \\
$\Gg\subseteq\Sa$ & goal set; episodes terminate on first visit \\
$H$ & label horizon: at most $H$ abstract labels and $H-1$ stochastic transitions \\
$\tau=(s_0,a_0,\dots,s_T)$ & a raw trajectory with $T\le H-1$ transitions; $\suc$ is the set of trajectories whose terminal state lies in $\Gg$ \\
$\phi:\Sa\times\Aa\to\Sigmaa$ & optional state--action abstraction; $\phi(\tau)$ is the induced string over $\Sigmaa$ \\
$g_\Sigma\in\Sigmaa$ & distinguished terminal symbol appended on first entry into $\Gg$, so $|\phi(\tau)|\le H$ \\
$u\Subseq v$ & $u$ is a (not necessarily contiguous) subsequence of $v$ \\
$\mu$ has full conditional support & every feasible successful trajectory has positive $\mu$-mass \\
\midrule
\multicolumn{2}{@{}l}{\textbf{Coverage, core, and family} (\Cref{def:core}, \Cref{lem:generate})}\\
\midrule
$\mu$ & reference (probe-induced) trajectory law; $p_0=\mu_0(\suc)>0$ \\
$c_\mu(u)$ & \emph{coverage}: $\Pr_{\tau\sim\mu}[u\Subseq\phi(\tau)\mid\tau\in\suc]$ \\
$\delta\in[0,1)$ & coverage slack: $u\in\HH^\delta$ iff $c_\mu(u)\ge 1-\delta$. \emph{Margin} is reserved for $m=\bar\delta-\delta_0$, the slack to a fixed target family (\Cref{cor:rate}) \\
$\HH^\delta_\phi(\mu)$ & \emph{high-coverage family}: $\{u:c_\mu(u)\ge 1-\delta\}$ (hereditary under $\Subseq$) \\
$\Core^\delta_\phi(\mu)$ & \emph{$\delta$-core}: the $\Subseq$-maximal \emph{frontier} $\max_\Subseq\HH^\delta_\phi(\mu)$, which generates the family \\
``generated by the core'' & $u\in\HH^\delta_\phi(\mu)$, i.e.\ $u\Subseq v$ for some frontier $v\in\Core^\delta_\phi(\mu)$ \\
$\phi$-bottleneck & symbol $\sigma$ with $c_\mu(\sigma)=1$; equivalently a step-cut of $G_{\suc}$ (\Cref{thm:exist}) \\
$G_{\suc}$ & action-labelled time-unrolled success graph: vertices $(s,t)$ and edges $(s,t)\xrightarrow{a}(s',t+1)$ from $\mu$-positive successes; each edge carries $\phi(s,a)$ \\
$o_\sigma$ & time-aware option $(I_\sigma,\pi_\sigma,\beta_\sigma)$ realizing $\sigma$: initiation set, policy, termination (\Cref{ass:exec}) \\
$V^\star,\,\mathcal P_i(t)$ & optimal almost-sure goal-reaching cost and admissible policies for landmark segment $i$ (\Cref{prop:chain}) \\
\midrule
\multicolumn{2}{@{}l}{\textbf{Estimation} (\Cref{sec:estimation})}\\
\midrule
$\widehat c_N(u)$ & empirical coverage from $N$ probe rollouts; $N_{\suc}$ successful among them \\
$t,\,\gamma,\,\rho$ & estimation tolerance, gap to the threshold $1-\delta$, and failure probability \\
\midrule
\multicolumn{2}{@{}l}{\textbf{Decentralized game and drift} (\Cref{sec:marl})}\\
\midrule
$\pi_{-1}^e,\,\pi_2^e$ & joint peer policy and its one-peer specialization at episode $e$; state-based and fixed within the episode \\
$\MarkovGame$ & stationary $n$-player Markov game; distinct from the goal set $\Gg$ \\
$\Mm_e,\,P_e,\,\mu_e$ & focal-induced MDP, kernel, and probe law at episode $e$ \\
$\mathcal C=(\Mm_e)_{e\ge0}$ & focal continual RL process induced by the peers (\Cref{prop:reduction}) \\
$\norm{P-Q}_{1,\infty}$ & rowwise difference $\sup_{s,a_1}\sum_{s'}|P(s'\mid s,a_1)-Q(s'\mid s,a_1)|$ \\
$\TV(\cdot,\cdot)$ & total variation distance \\
$\varepsilon_e$ & trajectory-law drift bound $\frac{1}{2}(H-1)\norm{P_e-P_0}_{1,\infty}$ (\Cref{lem:sim}) \\
$V_e$ & variation budget $\sum_{j=1}^{e}\norm{P_j-P_{j-1}}_{1,\infty}$ \\
$\delta'$ & drifted slack; $\delta'\le\delta+\frac{\varepsilon_e}{p_0}$ (\Cref{thm:stab}) \\
$\eta,\,c$ & peer step size and score-to-TV constant (\Cref{rem:pg}) \\
$\mathcal K_e^{\delta_{\mathrm{ref}}\to\bar\delta}$ & reference patterns from $\HH^{\delta_{\mathrm{ref}}}_\phi(\mu_0)$ that remain in the target family $\HH^{\bar\delta}_\phi(\mu_e)$ \\
$E^\star,\,E_{\mathrm{exit}}^{\bar\delta}(u)$ & certified survival horizon and first episode outside the target family (\Cref{prop:erosion,cor:rate}); membership holds through $E^\star$, so exit is at least $E^\star+1$ \\
$\Pi_2^+$ & peer policies with positive probe success, $\{\pi_2:\mu(\pi_2)(\suc)>0\}$ \\
$\operatorname{Supp}_{\suc}(\pi_2)$ & positive-mass successful trajectories under peer policy $\pi_2$ \\
$\HH_{\mathrm{rob}},\,\Corerob$ & robust (partner-universal) family $\bigcap_{\pi_2\in\Pi_2^+}\HH^0_\phi(\mu(\pi_2))$ and its frontier (\Cref{def:robust}) \\
\midrule
\multicolumn{2}{@{}l}{\textbf{Policy-value transfer} (\Cref{sec:value})}\\
\midrule
$\mu_{e,k},\,\nu_{e,k}$ & monitoring and deployment trajectory laws for supplied policy or option chain $\pi^k$ \\
$A_k,\,\Gamma_{e,k}$ & core-compatible success event and mass: $A_k=\suc\cap\{u_k\Subseq\phi(\tau)\}$, $\Gamma_{e,k}=\mu_{e,k}(A_k)=p_{e,k}c_{\mu_{e,k}}(u_k)$ \\
$G_{e,k},\,V_{e,k}$ & total trajectory return and deployment value $\E_{\nu_{e,k}}G_{e,k}$ \\
$\chi_{e,k},\,\xi_{e,k},\,b_{e,k}$ & shift in the probability of $A_k$, compatible-return error, and absolute return mass outside $A_k$ (\Cref{ass:calibration}) \\
$\zeta_{e,k}$ & performance-calibration radius $\xi_{e,k}+b_{e,k}+r_k\chi_{e,k}$ \\
$\widehat\Gamma_{e,k},\,L_{e,k}$ & empirical compatible mass and its calibrated lower confidence value (\Cref{thm:value}) \\
\midrule
\multicolumn{2}{@{}l}{\textbf{Domain (SRC-$L$-$N$, \Cref{app:exp})}}\\
\midrule
$L,\,N$ & corridor length and number of peers \\
$\alpha,\,q$ & per-peer clear probability; induced scalar $q=1-(1-\alpha)^N$ \\
$s_0,\,g$ & start and goal; $t_i,b_i$ top- and bottom-route positions \\
$\rho_{\mathrm{rock}}$ & binary rock flag; $\beta$ peer hold-probability (E5) \\
\bottomrule
\end{tabular}
\caption{Notation.}
\label{tab:notation}
\end{table}

\clearpage
\section{Full Proofs}\label{app:proofs}
\subsection{Standard tools and conventions}\label{app:tools}
These facts are standard background tools rather than assumptions on the Markov game.

\paragraph{Total variation and Markov kernels.}
For probability measures $p,q$ on a finite or countable space,
\[
\TV(p,q):=\sup_A|p(A)-q(A)|=\frac{1}{2}\norm{p-q}_1.
\]
Thus $|p(A)-q(A)|\le\TV(p,q)$ for every event $A$. A Markov kernel is an $\ell_1$ contraction:
\[
\norm{pK-qK}_1\le\norm{p-q}_1.
\]
Indeed, summing the absolute signed mass after applying $K$ and then using that each row of $K$ sums to one gives the right-hand side. Combining contraction with the triangle inequality yields the one-step comparison
\begin{equation}\label{eq:kernel-comparison}
\norm{pK-qL}_1
\le \norm{p-q}_1+\sup_x\norm{K(x,\cdot)-L(x,\cdot)}_1.
\end{equation}
For controlled kernels we use the rowwise norm
\[
\norm{P-Q}_{1,\infty}
:=\sup_{s,a}\sum_{s'}|P(s'\mid s,a)-Q(s'\mid s,a)|.
\]
\Cref{eq:kernel-comparison} is the elementary telescoping step behind the factor $H-1$ in trajectory simulation.

\paragraph{Independent products.}
Tensoring with a common probability factor leaves total variation unchanged, so
\[
\TV(p\otimes r,q\otimes r)=\TV(p,q).
\]
Replacing independent factors one at a time and applying the triangle inequality therefore gives
\begin{equation}\label{eq:product-tv}
\TV\!\left(\bigotimes_{j=1}^{m}p_j,\bigotimes_{j=1}^{m}q_j\right)
\le\sum_{j=1}^{m}\TV(p_j,q_j).
\end{equation}

\paragraph{Concentration and simultaneous events.}
If $N_{\suc}\sim\mathrm{Bin}(N,p_0)$, the multiplicative Chernoff bound at half the mean is
\begin{equation}\label{eq:chernoff-tool}
\Pr\!\left[N_{\suc}<\frac{Np_0}{2}\right]\le\exp\!\left(-\frac{Np_0}{8}\right).
\end{equation}
If $\widehat c$ is the mean of $n$ independent Bernoulli variables with mean $c$, Hoeffding's inequality gives
\begin{equation}\label{eq:hoeffding-tool}
\Pr\!\left[|\widehat c-c|>t\right]\le2\exp(-2nt^2).
\end{equation}
For events $E_1,\ldots,E_M$, the union bound $\Pr(\bigcup_iE_i)\le\sum_i\Pr(E_i)$ converts these pointwise statements into the simultaneous guarantees used for fixed candidate families and policy libraries. No independence across candidates or checkpoints is required for this last step.

\paragraph{Two-point testing.}
For $x,y\in(0,1)$, Bernoulli relative entropy satisfies
\begin{equation}\label{eq:bernoulli-kl-tool}
\mathrm{KL}\!\left(\mathrm{Bern}(x)\Vert\mathrm{Bern}(y)\right)
\le\frac{(x-y)^2}{y(1-y)}.
\end{equation}
For product observations, relative entropy tensorizes and Pinsker's inequality converts it to total variation:
\begin{equation}\label{eq:testing-tools}
\begin{aligned}
\mathrm{KL}(P^{\otimes N}\Vert Q^{\otimes N})
&=N\,\mathrm{KL}(P\Vert Q),\\
\TV(P,Q)&\le\sqrt{\frac{1}{2}\mathrm{KL}(P\Vert Q)}.
\end{aligned}
\end{equation}
Finally, every binary test $\psi$ obeys the two-point bound
\begin{equation}\label{eq:two-point-tool}
\max\!\left\{P(\psi=1),Q(\psi=0)\right\}
\ge\frac{1-\TV(P,Q)}{2}.
\end{equation}
These facts turn small information distance between two threshold-separated models into a necessary sample size.

\paragraph{Smoothness and logistic policies.}
If $f$ has an $L$-Lipschitz gradient, integration along the line segment from $x$ to $x+h$ gives the descent lemma
\begin{equation}\label{eq:descent-tool}
f(x+h)\le f(x)+\langle\nabla f(x),h\rangle+\frac{L}{2}\norm{h}_2^2.
\end{equation}
For the logistic map $\varsigma(x)=(1+e^{-x})^{-1}$,
\begin{equation}\label{eq:logistic-tool}
\varsigma'(x)=\varsigma(x)(1-\varsigma(x)),
\qquad |\varsigma'(x)|\le\frac{1}{4},
\end{equation}
so $\varsigma$ is $\frac{1}{4}$-Lipschitz. \Cref{eq:descent-tool} controls the second-order cost of a policy update, while \Cref{eq:logistic-tool} converts parameter movement into policy total variation in the analytic example.

The sharp conditioning relation used in \Cref{lem:condcov} is also established background: it specializes the total-variation model of \citet[Section~2.1 and Proposition~3.2]{pelessoni2022dilation}. The proof below makes the specialization and its equality case explicit.

\subsection{Proofs}

\paragraph{Proof of \Cref{lem:generate}.}
If $u\preceq w\preceq\phi(\tau)$ then $u\preceq\phi(\tau)$, so the occurrence event for $w$ is contained in that for $u$ and $c_\mu(u)\ge c_\mu(w)$. Thus $\HH^\delta_\phi(\mu)$ is hereditary. For any $u$ in this family, choose among its finitely many high-coverage extensions one of maximum length. It is maximal, and heredity places everything below every such maximal element back in the family. \hfill$\square$

\paragraph{Proof of \Cref{thm:exist}.}
\proofstep{Step 1: establish non-emptiness} By the goal-symbol convention every $\tau\in\suc$ satisfies $g_\Sigma\preceq\phi(\tau)$, so $c_\mu(g_\Sigma)=1\ge 1-\delta$. Thus $(g_\Sigma)$ belongs to the finite high-coverage family, which consequently has a $\preceq$-maximal element. This proves part (a).

\proofstep{Step 2: extend a bottleneck to the frontier} Let $\sigma$ be a $\phi$-bottleneck. Since $\sigma\ne g_\Sigma$ and $g_\Sigma$ is the unique terminal label, every positive-mass success contains $\sigma$ before $g_\Sigma$. Hence $(\sigma,g_\Sigma)$ is common. Extending this common sequence within the finite hereditary family yields a maximal element $u^\star$ with $(\sigma,g_\Sigma)\preceq u^\star$. In particular, $u^\star$ strictly extends $(g_\Sigma)$.

\proofstep{Step 3: identify the cut condition} Work in the action-labelled time-unrolled success graph $G_{\suc}$ defined before the theorem.  A path records its state, time, and action at every edge.  Each edge has positive transition probability, so the Markov property makes the product of its transition probabilities positive; an $s_0$--$\Gg$ path is therefore a feasible success.  Full conditional support assigns every such success positive $\mu$-mass.  Conversely, every positive-mass success traces one of these paths. If the $\sigma$-labelled edges are not a step-cut, a path avoids them and gives a positive-mass success without $\sigma$, hence $c_\mu(\sigma)<1$. Conversely, a step-cut forces every success to take a $\sigma$-labelled step. This proves the equivalence in part (b).

\proofstep{Step 4: assemble ordered bottlenecks} If $\sigma_1,\ldots,\sigma_m$ are order-consistent, each positive-mass success has indices $i_1<\cdots<i_m$ carrying them. All these labels precede the unique terminal label, so $(\sigma_1,\ldots,\sigma_m,g_\Sigma)$ is common. Any maximal common extension has length at least $m+1$. This proves part (c). \hfill$\square$

\paragraph{Why time indexing is necessary.}
The un-indexed union of success edges can splice transitions from different successes into an infeasible walk. With $H=4$ and successes $s_0\!\to\!a\!\to\!z\!\to\!g$, $s_0\!\to\!a\!\to\!b\!\to\!g$, $s_0\!\to\!b\!\to\!c\!\to\!g$, and $s_0\!\to\!c\!\to\!g$, label every success so that it contains $\sigma$. The spliced walk $s_0\!\to\!a\!\to\!b\!\to\!c\!\to\!g$ can avoid $\sigma$, but it uses four transitions and exceeds the budget $H-1=3$. Such walks are absent from $G_{\suc}$ because time is part of each vertex.

\paragraph{Proof of \Cref{prop:sound}.}
By \Cref{lem:generate}, $u\in\HH^0_\phi(\mu)$ means $c_\mu(u)=1$, i.e.\ $\mu(\{u\preceq\phi(\tau)\}\mid\suc)=1$, so every $\mu$-positive success contains $u$. Under the full-support hypothesis every optimal $\tau^\star$ has $\mu(\tau^\star)>0$, hence $u\preceq\phi(\tau^\star)$. Restricting the feasible set to trajectories containing all core elements therefore removes no optimal successful trajectory. Thus the core restriction is admissible for optimal successful-trajectory search; expected-return guarantees follow after executability calibration, as formalized in \Cref{thm:value}. \hfill$\square$

\paragraph{Proof of \Cref{prop:chain}.}
\proofstep{Step 1: establish goal-reaching} Assumption (A1) starts the first option, (A2) realizes every $\sigma_i$, (A3) chains consecutive options, and (A4)--(A5) terminate the chain in $\Gg$ within $H-1$ transitions. Thus the chain reaches the goal almost surely.

\proofstep{Step 2: expose the separating states} Use the comparator $V^\star$ and admissible segment classes $\mathcal P_i(t)$ defined in the proposition. Index each landmark by the state its labelled step \emph{enters}: let $g_i$ be the unique entered state for $\sigma_i$, put $g_0=s_0$, and let $g_{r+1}\in\Gg$ be entered by the final $g_\Sigma$-labelled step. By hypothesis, every success enters $g_1,\ldots,g_r$ in order, so each $(g_i,\cdot)$ separates $s_0$ from $\Gg$ in the time-unrolled success graph.  Consequently every policy admissible for $V^\star(g_{i-1},t)$ has a first arrival at $g_i$ and decomposes into an admissible segment in $\mathcal P_i(t)$ followed by a continuation from $(g_i,T_i)$.

\proofstep{Step 3: write the Bellman decomposition} Let $T_i$ be the arrival time at $g_i$ and $C_i$ the cost of segment $i$, including the final goal-labelled transition in segment $r+1$. Backward induction on the time-augmented state gives
\[
V^\star(g_{r+1},t)=0,\qquad V^\star(g_{i-1},t)=\inf_{\pi_i\in\mathcal P_i(t)}\E\big[\,C_i+V^\star(g_i,T_i)\,\big|\,g_{i-1},t\,\big]\quad (i=r{+}1,\dots,1).
\]
The separators prevent probability mass from bypassing a segment, so this is exactly the Bellman equation on the ordered-landmark structure. By the proposition's attainment premise, each supplied option realizes its segment infimum; composing them therefore realizes $V^\star(s_0,0)$. The downstream value appears once, inside the recursion. For deterministic segment durations the recursion reduces to $\sum_i d^\star(g_{i-1},g_i)$. \hfill$\square$

\paragraph{Why entered-state singletons are needed.}
Indexing by the state a labelled step leaves would place the terminal symbol outside $\Gg$ and omit the final transition cost. Likewise, if $\phi^{-1}(\sigma_i)$ can enter two states $g_i,g_i'$ with different continuation values, a locally cheapest preceding option can terminate at the state with the larger downstream cost. The resulting greedy chain is strictly suboptimal, with a gap that can be $\Omega(H)$.

\paragraph{Proof of \Cref{prop:order}.}
Construct a start state with two actions leading to two deterministic branches.  Label the first branch in order by $(x,y,g_\Sigma)$ and the second by $(y,x,g_\Sigma)$, and let the probe select both start actions with positive probability.  These are exactly the positive-mass successes.  Their maximal common subsequences are $(x,g_\Sigma)$ and $(y,g_\Sigma)$, which form an antichain; neither linear ordering of $x,y$ is common. \hfill$\square$

\paragraph{Proof of \Cref{prop:reduction}.}
\proofstep{Step 1: marginalize one transition} Fix episode $e$ and a focal history ending at $(s_t,a_{1,t})$. The peers sample their joint action from the stated Markov kernel; persistent randomness is part of the state. Hence
\[
\Pr(s_{t+1}=s'\mid s_t,a_{1,t})
=\sum_{a_{-1}}P(s'\mid s_t,a_{1,t},a_{-1})\pi_{-1}^e(a_{-1}\mid s_t)
=P_e(s'\mid s_t,a_{1,t}).
\]

\proofstep{Step 2: lift the identity to complete trajectories} The initial laws agree. If the two models agree on every length-$t$ focal history, then the focal policy gives the same next-action law and Step 1 gives the same next-state law; therefore they agree on every length-$(t+1)$ history. Induction proves equality of the complete focal trajectory laws. The same marginalization gives conditional stage reward $R_e(s_t,a_{1,t})$, so expected additive returns agree. Success, abstraction, subsequence occurrence, coverage, and the core are functions of this law and therefore transport exactly.

\proofstep{Step 3: translate peer drift into MDP drift} For every $(s,a_1)$,
\[
\begin{aligned}
\sum_{s'}|P_e(s'\mid s,a_1)-P_{e-1}(s'\mid s,a_1)|
&\le \sum_{a_{-1}}|\pi_{-1}^e(a_{-1}\mid s)-\pi_{-1}^{e-1}(a_{-1}\mid s)|\\
&=2\TV(\pi_{-1}^e(\cdot\mid s),\pi_{-1}^{e-1}(\cdot\mid s)).
\end{aligned}
\]
Taking the supremum proves the kernel bound. Similarly,
\[
|R_e(s,a_1)-R_{e-1}(s,a_1)|
\le R_{\max}\sum_{a_{-1}}|\pi_{-1}^e(a_{-1}\mid s)-\pi_{-1}^{e-1}(a_{-1}\mid s)|
\le 2R_{\max}D_e.
\]

\proofstep{Step 4: separate independent peers} If the peers act conditionally independently, their joint law is a product. Replacing one factor at a time and using
the product bound \eqref{eq:product-tv} gives
$\TV(\bigotimes_{j=2}^n\pi_j^e,\bigotimes_{j=2}^n\pi_j^{e-1})
\le\sum_{j=2}^n\TV(\pi_j^e,\pi_j^{e-1})$ pointwise in $s$, which proves the final claim. \hfill$\square$

\paragraph{Proof of \Cref{prop:impossible}.}
The Switch--Rock family (\Cref{app:exp}) has, at $\alpha=0$, exactly the top-route successes, whose focal-position string is $(s_0,t_1,t_2,g)$; so $\Core^0_\phi(\mu_0)=\{(s_0,t_1,t_2,g)\}$ and $t_1$ is generated ($t_1\preceq(s_0,t_1,t_2,g)$, a $\phi$-bottleneck by \Cref{thm:exist}(b)). For $0<\alpha\le\frac{1}{2}$ the peer clears the rock with positive probability, opening the bottom route, while holding also has positive probability, preserving top-route successes. Thus both route strings occur with positive mass, $c_{\mu_\alpha}(t_1)<1$, and the frontier collapses to $\{(s_0,g)\}$. The kernel bound $\|P_\alpha-P_0\|_{1,\infty}\le 2\sup_s\TV(\pi_2^\alpha,\pi_2^0)=2\alpha$ follows from the definition of the induced kernel and the identity $\|\sum_{a_2}P(\cdot\mid s,a_1,a_2)\Delta\pi_2(a_2)\|_1\le\sum_{a_2}|\Delta\pi_2(a_2)|=2\TV$. Since the indicator jumps from $1$ to $0$ at $\alpha=0^+$ while the drift $\to 0$, no modulus of continuity exists. \hfill$\square$

\paragraph{Proof of \Cref{lem:sim}.}
\proofstep{Step 1: place both processes on one history space} Pad terminated trajectories with an absorbing state. Let $\nu_t^e,\nu_t^0$ be the resulting history laws after $t$ transitions, put $d_t:=\|\nu_t^e-\nu_t^0\|_1$, and write $\Delta:=\|P_e-P_0\|_{1,\infty}$. Because the focal probe is fixed, the two history-extension kernels $K_e,K_0$ differ by at most $\Delta$ in rowwise $\ell_1$.

\proofstep{Step 2: telescope the one-step discrepancy} Applying the one-step kernel comparison \eqref{eq:kernel-comparison} gives
\[
\begin{aligned}
d_{t+1}
&=\|\nu_t^eK_e-\nu_t^0K_0\|_1\\
&\le \|\nu_t^eK_e-\nu_t^0K_e\|_1
   +\|\nu_t^0K_e-\nu_t^0K_0\|_1
\le d_t+\Delta .
\end{aligned}
\]
Since $d_0=0$, induction yields $d_{H-1}\le(H-1)\Delta$.

\proofstep{Step 3: convert $\ell_1$ distance to total variation} Using the total-variation identity stated in \Cref{app:tools} proves
$\|\mu_e-\mu_0\|_{\TV}\le\frac{1}{2}(H-1)\Delta=\varepsilon_e$. \hfill$\square$

\paragraph{Proof of \Cref{lem:condcov}.}
\emph{Connection to conditioning literature.} In the finite trajectory space, the probability laws in the total-variation model with reference $P_0=\mu$ and parameters $b=1$, $a=-\varepsilon$ are exactly those with $\TV(\nu,\mu)\le\varepsilon$. For $0<\varepsilon<\mu(A)$ and $\suc\setminus A\ne\varnothing$, Proposition~3.2 of \citet{pelessoni2022dilation}, with conditioning event $B=\suc$, gives the lower conditional envelope
\[
\inf_{\nu:\,\TV(\nu,\mu)\le\varepsilon}\nu(A\mid\suc)
=\frac{\mu(A)-\varepsilon}{\mu(\suc)}
=\mu(A\mid\suc)-\frac{\varepsilon}{p}.
\]
When $A=\suc$, the conditional probability is one; when $\varepsilon=0$, the laws coincide. For self-containment, the following direct proof covers all cases of the stated inequality.

Put $C:=\suc\setminus A$, $a:=\mu(A)$, and $c:=\mu(C)$, so $a+c=p$.

\proofstep{Step 1: bound the two success cells} Total variation gives
\[
\nu(A)\ge a-\varepsilon,\qquad \nu(C)\le c+\varepsilon.
\]
The assumption $\varepsilon<a$ implies $\nu(A)>0$, and hence $\nu(\suc)>0$.

\proofstep{Step 2: normalize only after the extremal move} The map $(x,y)\mapsto\frac{x}{x+y}$ is increasing in $x$ and decreasing in $y$. Therefore
\[
\nu(A\mid\suc)
=\frac{\nu(A)}{\nu(A)+\nu(C)}
\ge\frac{a-\varepsilon}{(a-\varepsilon)+(c+\varepsilon)}
=\frac{a}{p}-\frac{\varepsilon}{p}
=\mu(A\mid\suc)-\frac{\varepsilon}{p}.
\]

\proofstep{Step 3: show sharpness} On a three-cell space $(A,C,\suc^c)$, choose $\mu(A)>\varepsilon$ and obtain $\nu$ by moving exactly $\varepsilon$ mass from $A$ to $C$. The total-variation distance is $\varepsilon$, the success mass remains $p$, and equality holds above. Thus no smaller universal coefficient than $\frac{1}{p}$ is possible. \hfill$\square$

\paragraph{Proof of \Cref{thm:stab}.}
\proofstep{Step 1: move from kernels to trajectories} By \Cref{lem:sim},
$\TV(\mu_e,\mu_0)\le\varepsilon_e$.

\proofstep{Step 2: apply the sharp conditioning lemma} Let
$A:=\{u\preceq\phi(\tau)\}\cap\suc$. Then
$\mu_0(A)=p_0c_{\mu_0}(u)\ge(1-\delta)p_0>\varepsilon_e$.
Applying \Cref{lem:condcov} with $(\mu,\nu)=(\mu_0,\mu_e)$ gives
\[
c_{\mu_e}(u)\ge c_{\mu_0}(u)-\frac{\varepsilon_e}{p_0}
\ge(1-\delta)-\frac{\varepsilon_e}{p_0}.
\]

\proofstep{Step 3: return to the core family} Put $\delta':=1-c_{\mu_e}(u)$. The last display gives
$\delta'\le\delta+\frac{\varepsilon_e}{p_0}<1$, so
$u\in\HH^{\delta+\frac{\varepsilon_e}{p_0}}_\phi(\mu_e)$; \Cref{lem:generate} then places $u$ below a frontier element of the same family. Tightness follows from the mass transfer in Step 3 of \Cref{lem:condcov}. \hfill$\square$

\paragraph{Remark on the hypothesis.}
The condition $\varepsilon_e<(1-\delta)p_0$ is exact for the displayed non-vacuous conclusion. The weaker requirement $\varepsilon_e<p_0$ does not suffice: with $p_0=1$, $\delta=0.9$, and $\varepsilon_e=0.1$, coverage can fall to zero.

\paragraph{Proof of \Cref{prop:erosion}.}
Let $\mathcal K_e=\mathcal K_e^{\delta_{\mathrm{ref}}\to\bar\delta}$ as defined before the proposition. If $u\preceq v$, then the occurrence event for $v$ is contained in that for $u$, hence $c_{\mu_e}(u)\ge c_{\mu_e}(v)$ at every episode. Therefore $v\in\mathcal K_e$ and $u\preceq v$ imply $u\in\mathcal K_e$, so $\mathcal K_e$ is hereditary and has a finite maximal frontier. The same inequality implies that whenever $v$ remains above the target threshold, so does $u$; hence $E_{\mathrm{exit}}^{\bar\delta}(u)\ge E_{\mathrm{exit}}^{\bar\delta}(v)$. This argument is episode-wise and permits later re-entry. \hfill$\square$

\paragraph{Proof of \Cref{cor:rate}.}
\proofstep{Step 1: accumulate peer movement} The kernel triangle inequality and \Cref{prop:reduction} give
\[
\|P_e-P_0\|_{1,\infty}\le V_e
\le 2\sum_{j=1}^{e}\sup_s\TV(\pi_2^j,\pi_2^{j-1})
\le2c\eta e.
\]
Hence $\varepsilon_e\le c\eta(H-1)e$.

\proofstep{Step 2: spend the coverage margin} The initial slack is $\delta_0$ and the target slack is $\bar\delta$, so the available margin is $m=\bar\delta-\delta_0$. By \Cref{thm:stab}, membership in $\HH^{\bar\delta}$ is certified whenever
\[
\frac{\varepsilon_e}{p_0}\le m,\qquad\text{i.e.}\qquad
e\le\frac{mp_0}{c\eta(H-1)}.
\]
Because $m<1-\delta_0$, the premise $\varepsilon_e<(1-\delta_0)p_0$ also holds throughout this range.

\proofstep{Step 3: respect episode integrality} Taking the integer part gives the displayed $E^\star=\Omega(\frac{1}{\eta})$. Membership holds through episode $E^\star$, so $E_{\mathrm{exit}}^{\bar\delta}(u)\ge E^\star+1$. \hfill$\square$

\paragraph{Proof of \Cref{prop:p2example}.}
\proofstep{Step 1: compute the exact update direction} Since $c'(\theta)=-c(\theta)(1-c(\theta))$ and the divert reward exceeds the keep reward by $\Delta$,
\[
d_e=J'(\theta_e)=-\Delta c'(\theta_e)=\Delta c(\theta_e)(1-c(\theta_e))>0.
\]

\proofstep{Step 2: verify effective conflict} Before exit, $c(\theta_e)\in[\ell,c_0]$, and therefore
\[
-c'(\theta_e)d_e=\Delta[c(\theta_e)(1-c(\theta_e))]^2\ge\Delta r_\star^2.
\]
Thus (P2) holds termwise with $\lambda=\Delta r_\star^2$.

\proofstep{Step 3: verify boundedness and smoothness} We have $|d_e|\le\frac{\Delta}{4}$ and
$|c''(\theta)|=c(\theta)(1-c(\theta))|1-2c(\theta)|\le\frac{1}{4}$, establishing (P1) and (P3). By \eqref{eq:logistic-tool}, the logistic map is $\frac{1}{4}$-Lipschitz, so
$\TV(\pi_{\theta_{e+1}},\pi_{\theta_e})\le\frac{\eta|d_e|}{4}\le\frac{\Delta\eta}{16}$.

\proofstep{Step 4: invoke the lifetime theorem} Apply \Cref{thm:twosided} with $p_0=1$, $H-1=2$, $c=\frac{\Delta}{16}$, and $\lambda=\Delta r_\star^2$. Its step-size condition becomes $\eta\le\frac{64r_\star^2}{\Delta}$, and substitution gives the displayed bounds. \hfill$\square$

\paragraph{Proof of \Cref{lem:erode}.}
\proofstep{Step 1: control one update} The smoothness inequality \eqref{eq:descent-tool} with $h=\eta d_e$ gives
$c(\theta_{e+1})\le c(\theta_e)+\eta\langle\nabla c(\theta_e),d_e\rangle+\frac{L_c}{2}\eta^2\|d_e\|_2^2$.

\proofstep{Step 2: sum before first exit} Summing to $E$ and applying (P1)--(P2) yields
\[
c(\theta_E)\le c(\theta_0)-\eta E\left(\lambda-\frac{L_cG^2}{2}\eta\right)
\le c(\theta_0)-\frac{\lambda}{2}\eta E .
\]
\hfill$\square$

\paragraph{Proof of \Cref{thm:twosided}.}
\proofstep{Step 1: certify the lower lifetime} \Cref{cor:rate} gives
$E_{\mathrm{exit}}\ge\lfloor\frac{mp_0}{c\eta(H-1)}\rfloor+1$. Also $c_{\mu_0}(u)=1-\bar\delta+m>1-\bar\delta$, so the pattern is present initially.

\proofstep{Step 2: force an exit by cumulative erosion} Put $K=\lfloor\frac{2m}{\lambda\eta}\rfloor+1$. If $K<E_{\mathrm{exit}}$, then \Cref{lem:erode} gives
$c_{\mu_K}(u)\le c_{\mu_0}(u)-\frac{\lambda}{2}\eta K<c_{\mu_0}(u)-m=1-\bar\delta$.
This contradicts membership before the first exit. Hence $E_{\mathrm{exit}}\le K$ and the exit is finite.

\proofstep{Step 3: check the boundary episode} If $E_{\mathrm{exit}}>1$, the last pre-exit episode satisfies
\[
1-\bar\delta\le c_{\mu_{E_{\mathrm{exit}}-1}}(u)
\le c_{\mu_0}(u)-\frac{\lambda}{2}\eta(E_{\mathrm{exit}}-1),
\]
so $E_{\mathrm{exit}}-1\le\frac{2m}{\lambda\eta}$. The floor-plus-one form records integrality and retention at threshold equality. \hfill$\square$

\paragraph{Proof of \Cref{prop:exec-drift}.}
Membership drift is realized by the SRC clearing construction: any positive clearing probability gives the bottom route positive successful mass, so the top-only pattern leaves the support family. For executability drift, let the peer clear with probability zero and hold the top gate independently with probability $\beta>0$ at each attempt. Every successful trajectory still follows the top convention and contains the same minimal abstract string $(s_0,t_1,t_2,g_\Sigma)$; because a one-attempt success has positive probability for every $\beta>0$, the full support core and its frontier are unchanged. Now fix an option calibrated at $\beta=1$ that repeatedly attempts the gate until it opens, terminating after $k$ failed attempts. Independence gives its probability of reaching the next initiation set as
\[
1-(1-\beta)^k,
\]
which tends to zero as $\beta\downarrow0$. Thus the supplied chain can fail (A3)--(A5) with probability arbitrarily close to one while exact core membership and the frontier remain unchanged. \hfill$\square$

\paragraph{Proof of \Cref{thm:value}.}
Fix $(e,k)$ and abbreviate $A=A_k$, $\mu=\mu_{e,k}$, $\nu=\nu_{e,k}$, $G=G_{e,k}$, $r=r_k$, and $\Gamma=\Gamma_{e,k}$.

\proofstep{Step 1: decompose value around compatible success mass} Splitting on $A$ gives the exact identity
\[
V_{e,k}-r\Gamma
=\E_\nu[(G-r)\Ind[A]]
+\E_\nu[G\Ind[A^c]]
+r\big(\nu(A)-\mu(A)\big).
\]
The three terms are bounded respectively by $\xi_{e,k}$, $b_{e,k}$, and $r_k\chi_{e,k}$. The triangle inequality yields
$|V_{e,k}-r_k\Gamma_{e,k}|\le\zeta_{e,k}$.

\proofstep{Step 2: obtain the value floor} If $u_k\in\HH^\delta_\phi(\mu_{e,k})$ and $p_{e,k}\ge\underline p_k$, then
$\Gamma_{e,k}=p_{e,k}c_{\mu_{e,k}}(u_k)\ge\underline p_k(1-\delta)$. The lower side of Step 1 gives the displayed floor.

\proofstep{Step 3: construct simultaneous value intervals} Since $\Ind[A_k]$ is Bernoulli with mean $\Gamma_{e,k}$, \eqref{eq:hoeffding-tool} and a union bound over $\mathcal I$ imply, with probability at least $1-\rho$, that
$|\widehat\Gamma_{e,k}-\Gamma_{e,k}|\le\beta_{e,k}$ for every $(e,k)\in\mathcal I$. Work on this event and put
$q_{e,k}:=r_k\beta_{e,k}+\zeta_{e,k}$. The value approximation implies
\[
L_{e,k}=r_k\widehat\Gamma_{e,k}-q_{e,k}
\le V_{e,k}
\le r_k\widehat\Gamma_{e,k}+q_{e,k}
=L_{e,k}+2q_{e,k}.
\]

\proofstep{Step 4: compare the selected library element} Since $\widehat k_e$ maximizes $L_{e,k}$, every comparator $k$ satisfies
\[
V_{e,k}
\le L_{e,k}+2q_{e,k}
\le L_{e,\widehat k_e}+2q_{e,k}
\le V_{e,\widehat k_e}+2q_{e,k}.
\]
Rearranging proves the comparator-wise guarantee. Taking $k=k_e^\star$ and summing proves cumulative transfer regret. The simultaneous event, rather than checkpoint independence, drives this argument; only the monitoring rollouts within each fixed $(e,k)$ evaluation need the stated law. \hfill$\square$

\paragraph{Proof of \Cref{cor:value-survival}.}
\Cref{cor:rate} gives $u_k\in\HH^{\bar\delta}_\phi(\mu_{e,k})$ for every $e\le E^\star$. Apply the value floor in \Cref{thm:value} with $\delta=\bar\delta$, $p_{e,k}\ge\underline p_k$, and $\zeta_{e,k}\le\bar\zeta_k$. For the endogenous floor, total variation gives $p_{e,k}\ge p_{0,k}-\varepsilon_{e,k}$; the certified horizon ensures $\varepsilon_{e,k}\le mp_{0,k}$, and $0<m<1$. \hfill$\square$

\paragraph{Proof of \Cref{prop:overlap}.}
If $u\in\HH_{\mathrm{rob}}$ then $c_{\mu(\pi_2)}(u)=1$ for every $\pi_2\in\Pi_2^+$, in particular for success-enabling $\pi_2^e,\pi_2^{e+1}$, so $u\in\HH^0_\phi(\mu_e)\cap\HH^0_\phi(\mu_{e+1})$; this gives the stated inclusion.  A pattern has coverage one under every visited peer exactly when it occurs in every trajectory in
$\bigcup_e\operatorname{Supp}_{\suc}(\pi_2^e)$.  Under the displayed union-of-supports equality in the proposition, this is equivalent to occurring in every trajectory in
$\bigcup_{\pi_2\in\Pi_2^+}\operatorname{Supp}_{\suc}(\pi_2)$, which is equivalent to coverage one under every $\pi_2\in\Pi_2^+$.  Therefore
$\bigcap_e\HH^0_\phi(\mu_e)=\HH_{\mathrm{rob}}$. Taking $\max_\preceq$ of both sides gives $\Corerob$. \hfill$\square$

\paragraph{Proof of \Cref{thm:estimation}(a)--(b) (upper bounds).}
\proofstep{Step 1: guarantee enough successful probes} Since $N_{\suc}\sim\mathrm{Bin}(N,p_0)$, \eqref{eq:chernoff-tool} gives
$\Pr[N_{\suc}<\frac{Np_0}{2}]\le e^{-\frac{Np_0}{8}}\le\frac{\rho}{2}$ under the theorem's second summand. Conditional on $N_{\suc}=n\ge \frac{Np_0}{2}$, the occurrence indicators on successful trajectories are i.i.d.\ $\mathrm{Bernoulli}(c_\mu(u))$. Hoeffding therefore gives
\[
\Pr\!\left[|\widehat c_N(u)-c_\mu(u)|>t\mid N_{\suc}=n\right]
\le 2e^{-2nt^2}\le2e^{-Np_0t^2}\le\frac{\rho}{2}.
\]

\proofstep{Step 2: combine the two deviations} A union bound proves (a).

\proofstep{Step 3: lift to $M$ fixed candidates} Apply part (a) with tolerance $\frac{\gamma}{2}$ and failure probability $\frac{\rho}{M}$ to each candidate, then union-bound over candidates. Every estimate is then on the same side of the threshold as its population coverage, proving (b). \hfill$\square$

\paragraph{Proof of \Cref{thm:estimation}(c) (lower bound).}
\proofstep{Step 1: construct two threshold-separated laws} Fix $a\in(0,\frac{1}{2}]$, assume $1-\delta\in[a,1-a]$ and $0<\gamma\le \frac{a}{2}$, and put $c_-=1-\delta-\gamma$, $c_+=1-\delta+\gamma$. Under single-probe laws $P,Q$, let a trajectory be unsuccessful with probability $1-p_0$ and otherwise let $\Ind[u\preceq\phi(\tau)]\sim\mathrm{Bernoulli}(c_\pm)$. These laws require opposite membership decisions.

\proofstep{Step 2: bound their information distance} Only the successful branch differs, so
\[
\mathrm{KL}(P\Vert Q)=p_0\,\mathrm{KL}\big(\mathrm{Bern}(c_-)\Vert\mathrm{Bern}(c_+)\big)\le p_0\,\frac{(c_--c_+)^2}{c_+(1-c_+)}=\frac{4p_0\gamma^2}{(1-\delta+\gamma)(\delta-\gamma)},
\]
using \eqref{eq:bernoulli-kl-tool}. The interior conditions give $c_+(1-c_+)\ge\kappa_a:=\frac{a}{2}(1-\frac{a}{2})$, hence $\mathrm{KL}(P\Vert Q)\le \frac{4p_0\gamma^2}{\kappa_a}$.

\proofstep{Step 3: reduce estimation to testing} The tools in \eqref{eq:testing-tools} yield
$\mathrm{TV}(P^{\otimes N},Q^{\otimes N})\le\sqrt{\frac{2Np_0\gamma^2}{\kappa_a}}$. By \eqref{eq:two-point-tool}, error below $\frac{1}{3}$ requires $\mathrm{TV}\ge\frac{1}{3}$. Therefore $N\ge\frac{\kappa_a}{18p_0\gamma^2}=\Omega_a(\frac{1}{p_0\gamma^2})$. \hfill$\square$

\section{Monitoring Certificates and Windowed Erosion}\label{app:monitoring}
\subsection{A certificate with estimated success mass}
A small reference success probability amplifies drift, so replacing $p_0$ by an upward-biased estimate can invalidate a certificate. The following construction budgets success-mass and conditional-coverage uncertainty together.

\begin{proposition}[Finite-sample survival certificate]\label{prop:empirical-certificate}
Fix a candidate $u$, reference law $\mu_0$ with $p_0=\mu_0(\suc)>0$, and $N\ge1$ independent reference rollouts. Let $n$ be their number of successes, $\widehat p=\frac{n}{N}$, and $\widehat c$ their conditional coverage estimate. Choose $\rho_p,\rho_c>0$ with $\rho_p+\rho_c<1$ and define
\[
\underline p=\max\!\left\{0,\widehat p-\sqrt{\frac{\log(\frac{1}{\rho_p})}{2N}}\right\},
\qquad
\underline c=
\begin{cases}
\max\!\left\{0,\widehat c-\sqrt{\frac{\log(\frac{2}{\rho_c})}{2n}}\right\},&n>0,\\
0,&n=0.
\end{cases}
\]
Let $\TV(\mu_e,\mu_0)\le\varepsilon_e$ be a valid drift bound on a common trajectory space with fixed success event and abstraction. With probability at least $1-\rho_p-\rho_c$, simultaneously at every episode with such a bound,
\[
\underline p>0,\qquad
\underline c-\frac{\varepsilon_e}{\underline p}\ge1-\bar\delta
\quad\Longrightarrow\quad
\mu_e(\suc)>0,\quad u\in\HH_\phi^{\bar\delta}(\mu_e),
\qquad \bar\delta<1 .
\]
If $\underline p=0$, the rule issues no certificate.
\end{proposition}

\begin{proof}
\proofstep{Step 1: control the two estimates}
One-sided Hoeffding gives $\underline p\le p_0$ except on an event of probability $\rho_p$. Conditional on $n>0$, the successful-rollout indicators are independent Bernoulli variables with mean $c_{\mu_0}(u)$; Hoeffding gives $\underline c\le c_{\mu_0}(u)$ except with probability at most $\rho_c$. For $n=0$ this inequality is automatic. A union bound controls both estimates, without assuming that they are independent.

\proofstep{Step 2: verify the conditioning lemma's domain}
On this event, whenever the rule certifies,
\[
\varepsilon_e\le\underline p(\underline c-1+\bar\delta)
<\underline p\,\underline c\le p_0c_{\mu_0}(u).
\]
Thus \Cref{lem:condcov} applies and gives
$c_{\mu_e}(u)\ge c_{\mu_0}(u)-\frac{\varepsilon_e}{p_0}
\ge\underline c-\frac{\varepsilon_e}{\underline p}\ge1-\bar\delta$.
The same estimation event works for every episode whose drift bound is valid.
\end{proof}

For $M$ candidates fixed before evaluation, allocate $\rho_c$ across the $M$ conditional-coverage bounds; the success-mass bound is shared. Data-dependent candidate discovery requires separate evaluation data or a uniform bound over its search class. An estimated drift bound requires its own confidence allocation; a union bound then adds its failure probability.

\subsection{Windowed erosion and forward prediction}
Write $[x]_+:=\max\{0,x\}$. The experimental erosion statistic is a positive coverage drop across a window. Its relation to cumulative directional progress is explicit.

\begin{proposition}[Windowed erosion]\label{prop:window-erosion}
Let $\theta_{j+1}=\theta_j+\eta d_j$, $\|d_j\|_2\le G$, and let $c$ have $L_c$-Lipschitz gradient on each update segment for $q\le j<t$, where $q=t-w$. Put $a_j=-\langle\nabla c(\theta_j),d_j\rangle$ and $B_w=\frac{L_cG^2}{2}\eta^2w$. Then
\[
\left|c(\theta_q)-c(\theta_t)-\eta\sum_{j=q}^{t-1}a_j\right|\le B_w,
\qquad
\left|[c(\theta_q)-c(\theta_t)]_+
-\left[\eta\sum_{j=q}^{t-1}a_j\right]_+\right|\le B_w .
\]
If this window has $\frac{1}{w}\sum_{j=q}^{t-1}a_j\ge\lambda>0$ and $\eta\le\frac{\lambda}{L_cG^2}$, its erosion is at least $\frac{\lambda}{2}\eta w$.
For monitored values $C_q,C_t$ satisfying $|C_j-c(\theta_j)|\le r_j$ at the two endpoints,
\[
\left|[C_q-C_t]_+-[c(\theta_q)-c(\theta_t)]_+\right|
\le r_q+r_t.
\]
In particular, on these events the observed score is at least
$\frac{\lambda}{2}\eta w-r_q-r_t$ under the window-progress premise.
\end{proposition}

\begin{proof}
\proofstep{Step 1: isolate each update's remainder}
The two-sided smoothness bound gives
$c(\theta_{j+1})=c(\theta_j)-\eta a_j+R_j$ with
$|R_j|\le\frac{L_c}{2}\eta^2G^2$.
Telescoping from $q$ to $t$ proves the first inequality.

\proofstep{Step 2: pass to the monitored statistic}
The positive-part map is $1$-Lipschitz. Applying it to the first inequality proves the second; applying it to the two endpoint errors proves the observation bound. Window progress and the step-size restriction imply
$\eta\sum a_j-B_w\ge\lambda\eta w-\frac{\lambda}{2}\eta w$, which proves the lower bound.
\end{proof}

For confidence intervals $[L_j,U_j]$ containing $c(\theta_j)$, one can take $C_j=L_j$ and $r_j=U_j-L_j$. A lower bound alone, without control of its width, does not control changes in that bound. This applies to confidence intervals with their stated joint coverage; the Wilson intervals used in the experiments are nominal rather than a replacement for the finite-sample construction above.

The prefix condition (P2) starts at episode zero and does not imply positive progress on every later window. Moreover, an observed past drop by itself does not guarantee a future drop. If the progress condition instead holds on every future prefix starting at $t$ until exit, the cumulative-erosion proof restarts at $t$: with current margin $m_t=c(\theta_t)-(1-\bar\delta)>0$, first exit occurs by
$t+\left\lfloor\frac{2m_t}{\lambda\eta}\right\rfloor+1$.
Thus the empirical prediction tests examine whether past erosion is informative about subsequent performance; the deterministic exit bound additionally requires future progress.

\subsection{Event-specific calibration and performance loss}
The value decomposition uses only $\nu_{e,k}(A_k)-\mu_{e,k}(A_k)$, rather than the discrepancy on every trajectory event. If independent monitoring and deployment samples of sizes $N_\mu,N_\nu$ give event frequencies $\widehat a_\mu,\widehat a_\nu$, then with probability at least $1-\rho_\mu-\rho_\nu$,
\[
|\nu(A_k)-\mu(A_k)|
\le
|\widehat a_\nu-\widehat a_\mu|
+\sqrt{\frac{\log(\frac{2}{\rho_\nu})}{2N_\nu}}
+\sqrt{\frac{\log(\frac{2}{\rho_\mu})}{2N_\mu}} .
\]
This is two applications of Hoeffding and the triangle inequality. The right-hand side, clipped at one, supplies $\chi_{e,k}$. To use estimated calibration bounds in \Cref{thm:value}, budget their simultaneous failure probability in addition to the theorem's monitoring error. The execution and off-event reward bounds $\xi_{e,k},b_{e,k}$ remain separate requirements.

\begin{corollary}[Coverage loss and value loss]\label{cor:erosion-value}
Fix one policy, candidate, and reward scale $r\ge0$ at two episodes $q,t$. Under \Cref{ass:calibration}, abbreviate their coverages by $c_q,c_t$, success masses by $p_q,p_t$, and radii by $\zeta_q,\zeta_t$. Then
\[
\left|(V_q-V_t)-r(p_qc_q-p_tc_t)\right|\le\zeta_q+\zeta_t.
\]
In particular,
\[
V_q-V_t\ge r p_q(c_q-c_t)-r|p_q-p_t|-\zeta_q-\zeta_t.
\]
\end{corollary}
\begin{proof}
Subtract the two value approximations in \Cref{thm:value} and apply the triangle inequality. Next use
$p_qc_q-p_tc_t=p_q(c_q-c_t)+(p_q-p_t)c_t$
and $0\le c_t\le1$.
\end{proof}
A coverage drop therefore certifies a value drop when its compatible-return contribution exceeds the success-mass change and calibration errors. This identifies the additional measurements needed to turn an erosion signal into a policy-performance certificate.

\section{Experimental Details}\label{app:exp}
\paragraph{Switch--Rock Corridor (SRC).} The joint state is $z=(f,\rho_{\mathrm{r}})$ with focal position $f\in\{s_0,t_1,t_2,b_1,b_2,g\}$ and a binary rock flag $\rho_{\mathrm{r}}\in\{0,1\}$ ($12$ states). Focal actions are $\{\text{up},\text{down},\text{fwd},\text{wait}\}$; peer actions are $\{\text{hold},\text{clear},\text{idle}\}$. Let $\text{hold}=\1[a_2{=}\text{hold}]$, $\text{clear}=\1[a_2{=}\text{clear}]$, and $\rho_{\mathrm{r}}'=\rho_{\mathrm{r}}\vee\text{clear}$. Transitions are deterministic given $(a_1,a_2)$:
\begin{itemize}[topsep=2pt,itemsep=0pt]
\item $f{=}s_0$: up$\to t_1$, down$\to b_1$, else stay.
\item $f{=}t_1$: fwd $\wedge$ hold $\wedge\ \rho_{\mathrm{r}}{=}0\to t_2$, else stay. \emph{(top gate; needs the rock intact)}
\item $f{=}t_2$: fwd$\to g$.
\item $f{=}b_1$: fwd $\wedge\ \rho_{\mathrm{r}}{=}1\to b_2$, else stay. \emph{(bottom gate; needs the rock cleared)}
\item $f{=}b_2$: fwd$\to g$; \quad $f{=}g$: absorbing.
\end{itemize}
Because reaching $t_2$ requires $\rho_{\mathrm{r}}{=}0$ while reaching $b_2$ requires $\rho_{\mathrm{r}}{=}1$, and $\rho_{\mathrm{r}}$ is monotone, the two routes are mutually exclusive: a single episode yields a top \emph{or} a bottom success, never both. The abstraction is $\phi(z,\cdot)=f$, so $\Sigmaa=\{s_0,t_1,t_2,b_1,b_2,g\}$; the robust core (\Cref{def:robust}) is generated by $(s_0,g)$, i.e.\ only the start and goal lie on every success under every success-enabling peer.

\paragraph{Generalized domain SRC-$L$-$N$.} The corridor length generalizes to $L$ (top positions $t_1,\dots,t_L$, bottom positions $b_1,\dots,b_L$), and the single peer generalizes to $N$ i.i.d.\ peers each clearing with probability $\alpha$. The $N$ peers enter the focal-induced kernel only through the scalar $q=1-(1-\alpha)^N$ (the probability that at least one peer clears), so the induced kernel remains exact and of size independent of $N$. This yields $\norm{P_q-P_{q'}}_{1,\infty}=2|q-q'|$ and, at $q{=}0$, a top core generated by the single frontier string $(s_0,t_1,\dots,t_L,g)$, whose high-coverage family contains $\Theta(L)$ length-1 patterns. The base SRC is recovered at $L{=}2,N{=}1$, and the reported $\frac{1}{N}$ law follows analytically from this scalar coupling.

\paragraph{Probe policy and core estimation.} The base studies use $H=12$ labels (at most $11$ transitions); scaling uses $H=3L+4$. Unless stated otherwise the focal \emph{probe} splits equally at $s_0$ and advances with probability $0.9$ elsewhere; this fixed measurement policy makes observed drift attributable to the peer (\Cref{sec:single-stab}). Coverage of a length-1 pattern $\sigma$ is the fraction of successful probe rollouts whose focal-position string contains $\sigma$; the estimated high-coverage family is $\{\sigma:c\ge 1-\delta\}$, with length-2 patterns evaluated analogously for \Cref{prop:order}. This is the bounded-length evaluator for the hereditary family $\HH^\delta$ generated by the core (\Cref{lem:generate}), and \Cref{thm:estimation} gives its sample complexity. Estimates use $N\ge 5{,}000$ rollouts (E1: $4\times10^4$; E2: $4\times10^4$; E3: $7\times10^3$ per measurement; E4: $5\times10^3$).

E4 additionally computes population coverage \emph{exactly}, conditional on each realized peer-parameter path. Let $x_t(z,h)$ be the probability of SRC state $z$ after $t$ transitions with hit bit $h\in\{0,1\}$ indicating whether $t_1$ has appeared as a source-state label on a completed decision step. The update marks the state being departed, matching $\phi(z,a_1)=f(z)$. Starting from $x_0((s_0,0),0)=1$, the forward recursion is
\[
x_{t+1}(z',h')
=\sum_{z,h,a_1,a_2}x_t(z,h)\pi_1(a_1\mid z)\pi_2^e(a_2\mid z)
P(z'\mid z,a_1,a_2)\,
\1\!\left[h'=h\vee\1[f(z){=}t_1]\right].
\]
After $H-1$ transitions,
\[
c_{\mu_e}(t_1)=
\frac{\sum_{z:f(z)=g}x_{H-1}(z,1)}
{\sum_{z:f(z)=g}\sum_{h=0}^1x_{H-1}(z,h)}.
\]
Population coverage is evaluated by this exact finite-state recursion.

\paragraph{Peer families.}
E1--E2 use $\pi_2^\alpha=\alpha\cdot\text{clear}+(1-\alpha)\cdot\text{hold}$ at every state, so $\alpha$ is a scalar drift knob and $\|P_\alpha-P_0\|_{1,\infty}=2\alpha$. E3 trains two independent tabular $Q$-learners ($\varepsilon{=}0.2$, step $0.2$, $\gamma{=}1$, terminal reward $1$, focal step penalty $0.02$, \emph{no} private peer reward) for $8{,}000$ episodes; the peer's convention is thus determined only by symmetry breaking. E4 uses a softmax policy-gradient peer (REINFORCE with a mean baseline, batch $192$) initialized to hold, with step sizes $\eta\in\{0.25,0.5,1,2,4\}$. Finite horizon, bounded return, and bounded softmax scores give a deterministic bound on each batch gradient, hence the per-update condition $\sup_s\TV(\pi_2^e,\pi_2^{e-1})\le c\eta$ (\Cref{rem:pg}). Population exit is evaluated exactly after every update; the finite-rollout exit is checked every $15$ updates using an independent evaluation RNG, so evaluation does not alter the subsequent training path. E5 uses a hold-probability $\beta$ peer that never clears (both routes' membership therefore unchanged) and a fixed option chain calibrated at $\beta{=}1$. E6 (single-agent control) fixes the peer to hold and $Q$-learns only the focal agent. The co-adaptation experiments (E3, E4) and all scaling experiments average over $10$ seeds.

\clearpage
\paragraph{Experimental suites.}
\begin{center}
\small
\begin{tabular}{@{}>{\raggedright\arraybackslash}p{0.28\linewidth} >{\raggedright\arraybackslash}p{0.12\linewidth} >{\raggedright\arraybackslash}p{0.52\linewidth}@{}}
\toprule
Suite & Exp. & Evidence \\
\midrule
Structural degradation and lifetime & E1--E6 & Exact membership discontinuity, zero stability-bound violations across 61 drift levels, persistent robust structure under convention churn, $R^2=0.999986$ inverse-rate population lifetime, executability separation, and stable single-agent control. \\
Rate and sharpness & E7--E9 & $N^\star p_0\gamma^2$ stays within $1.21\times$ across 15 cells; the trajectory-law ratio reaches $0.996$, and the kernel-family limit reaches $\frac{255}{256}=0.9961$. \\
Registered downstream confirmation & E10, E12 & Continual-control AP advantage $0.237$ and reward advantage $0.537$; cue-MNIST AP and replay-reward advantages $0.039$ and $0.0117$. \\
Exploratory peer-induced analysis & E11 & Post-hoc learned-peer LBF erosion AP advantage $0.228$ across eight development pairings. \\
\bottomrule
\end{tabular}
\end{center}

\paragraph{Reproducibility.} The supplement includes \texttt{src.py} (base domain and exact population evaluator), \texttt{src\_scaling.py} (generalized SRC-$L$-$N$), and the experiment scripts. The command \texttt{python3 code/regenerate\_all\_figs.py} reruns the SRC evaluations and regenerates downstream statistics plots from the supplied E10--E12 records. Fixed task illustrations are supplied as assets. Per-seed records and threshold conventions accompany the figures; \Cref{app:parameters} describes the artifact scope.

\clearpage
\subsection{Generalized domain and figures}
\Cref{fig:env} shows the environment and its SRC-$L$-$N$ generalization. Experiments E1--E6 (\Cref{fig:ramp,fig:bound,fig:iql,fig:pg,fig:exec}) test the single-instance predictions; the core-effect experiments (\Cref{fig:coreeffect}) test soundness; the $L\times N$ scaling (\Cref{fig:surface,fig:survival,fig:scalebound}) and core-computation cost (\Cref{fig:corecost}) test the scaling and complexity laws.

\begin{figure}[h]
\centering
\includegraphics[width=0.60\linewidth]{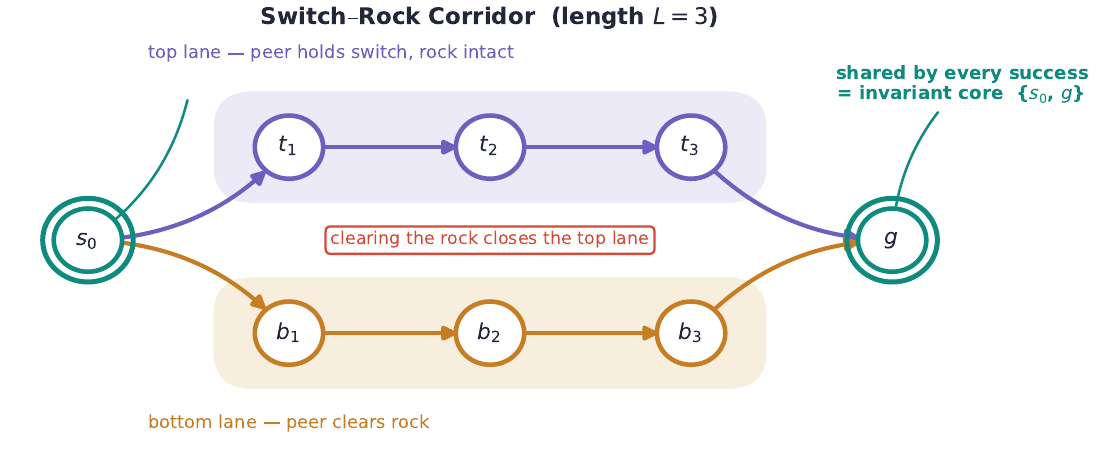}\hfill
\raisebox{0.2em}{\includegraphics[width=0.36\linewidth]{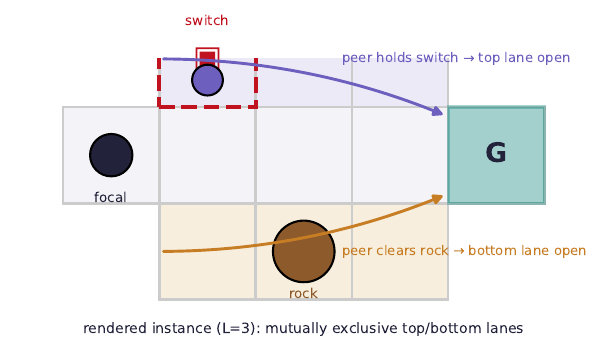}}
\caption{\textbf{The Switch--Rock Corridor.} \emph{Left:} transition-structure schematic---two mutually exclusive routes (top needs the switch held and rock intact; bottom needs the rock cleared) sharing only $s_0$ and $g$; clearing the rock disables the top gate. \emph{Right:} a rendered gridworld instance ($L{=}3$).}
\label{fig:env}
\end{figure}

\paragraph{E1: impossibility and repair.} Sweeping the peer's clear-probability $\alpha$ from $0$, the support mass of $t_1$ decays continuously, but coverage-1 membership of $t_1$ (equivalently, whether $t_1$ is generated by the $\delta{=}0$ core) is lost at the first $\alpha>0$ tested ($\alpha=2\times10^{-4}$), exactly as \Cref{prop:impossible} predicts, while $\delta$-families lose it only once the drift crosses $\delta$-ordered thresholds ($\alpha\approx2.3\times10^{-3},1.3\times10^{-2},6.0\times10^{-2}$ for $\delta=0.02,0.1,0.3$). The robust core $(s_0,g)$ is generated throughout.
\begin{figure}[h]
\centering
\includegraphics[width=0.92\linewidth]{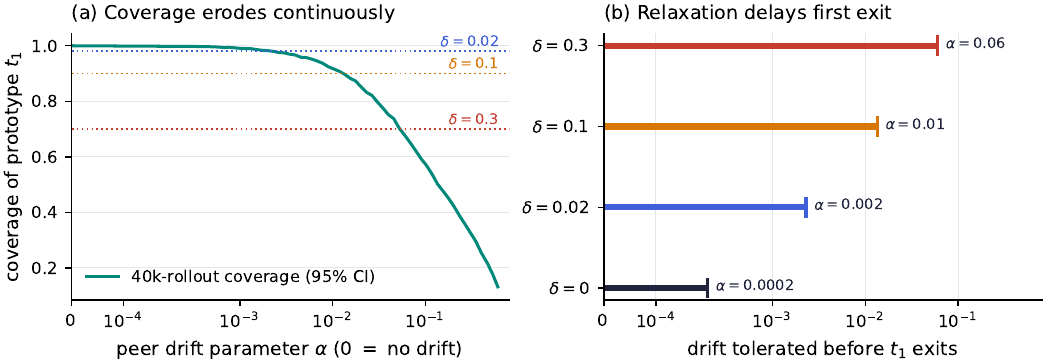}
\caption{\textbf{E1.} \captiona~Prototype mass decays continuously in $\alpha$; the band is an approximate 95\% conditional-binomial interval from $40{,}000$ rollouts. \captionb~Directly labeled first-exit drift: coverage-1 membership is lost at any $\alpha>0$, whereas $\delta$-families persist to $\delta$-ordered thresholds (\Cref{prop:impossible,thm:stab}).}
\label{fig:ramp}
\end{figure}

\paragraph{E2: the stability bound holds.} Across $61$ drift levels the empirical coverage of $t_1$ never falls below the \Cref{thm:stab} guarantee $1-\delta-\frac{\varepsilon_e}{p_0}$ (0 violations), with median slack $0.132$. The slack is smallest at $\alpha{=}0$, where guarantee and coverage coincide at $1$.
\begin{figure}[h]
\centering
\includegraphics[width=0.42\linewidth]{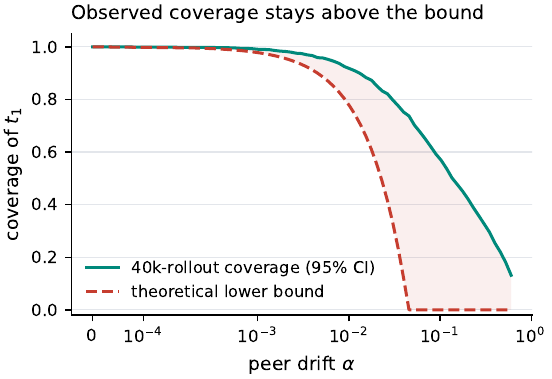}
\caption{\textbf{E2.} Empirical coverage (approximate 95\% conditional-binomial band, $40{,}000$ rollouts) stays above the \Cref{thm:stab} guarantee $1-\delta-\frac{\varepsilon_e}{p_0}$ at all $61$ drift levels (0 violations); the guarantee is clipped at $0$ where it goes vacuous.}
\label{fig:bound}
\end{figure}

\paragraph{E3: robust core is invariant while conventions churn.} Two independent $Q$-learners co-adapt; the peer has no private reward, so convention choice is pure symmetry breaking. Over $8{,}000$ episodes the preferred convention flips $\approx126$ times on average (10 seeds), the prototypes $t_1,b_1$ oscillate, yet the robust core $(s_0,g)$ is generated in $100\%$ of measurements across all seeds (\Cref{prop:overlap}). The single-agent control shows the estimated family equal to the truth with $0$ changes (\Cref{sec:single-stab}).
\begin{figure}[h]
\centering
\includegraphics[width=0.92\linewidth]{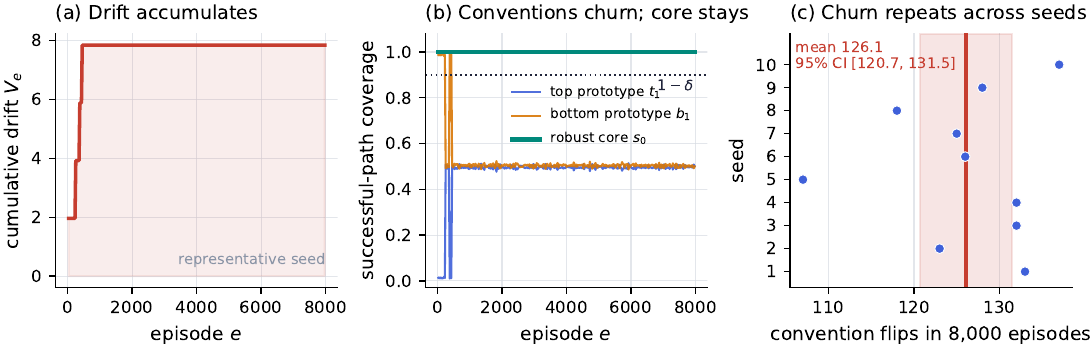}
\par\vspace{0.2em}
\includegraphics[width=0.48\linewidth]{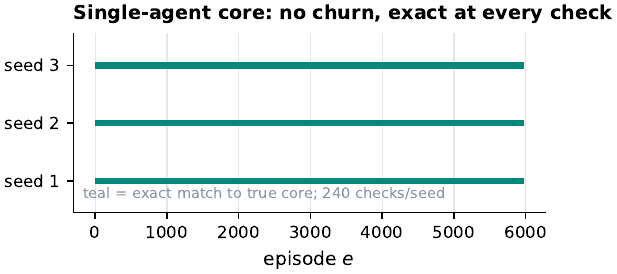}
\caption{\textbf{E3.} \emph{Top:} a representative seed shows cumulative kernel drift (a) and conventional-prototype churn while the robust core stays at mass $1$ (b); seed-level flip counts with a 95\% CI show that churn repeats across all 10 seeds (c). \emph{Bottom:} in the single-agent control, every one of 240 checks in each of three seeds exactly matches the true core.}
\label{fig:iql}
\end{figure}

\paragraph{E4: exact population survival tracks $\frac{1}{\eta}$.} A softmax policy-gradient peer with step size $\eta$ erodes $t_1$. Using the exact recursion above after every update, the mean population first-exit updates are $127.9,64.9,33.0,17.1,9.6$ for $\eta=0.25,0.5,1,2,4$ (10 seeds). They are linear in $\frac{1}{\eta}$ with $R^2=0.999986$; the products $\eta E_{\mathrm{exit}}$ range from $32.0$ to $38.4$, with the largest-step deviation attributable in part to integer-time resolution. This is an exact population statistic evaluated along each stochastic training path.

At the scheduled $15$-update evaluations, the $5{,}000$-rollout coverage estimate has RMSE $0.0068$ against exact coverage, and its first-exit update matches the exact grid first exit in all $50$ seed--step-size runs. Its mean absolute exit-time delay relative to the update-wise population exit is $8.9$ updates, exactly equal here to the delay from checking only on the $15$-update grid: there are no additional Monte-Carlo crossing errors. Thus E4 isolates inverse-step-size learning from finite-evaluation resolution and realizes the population lifetime law in the smooth SRC peer family.
\begin{figure}[h]
\centering
\includegraphics[width=0.92\linewidth]{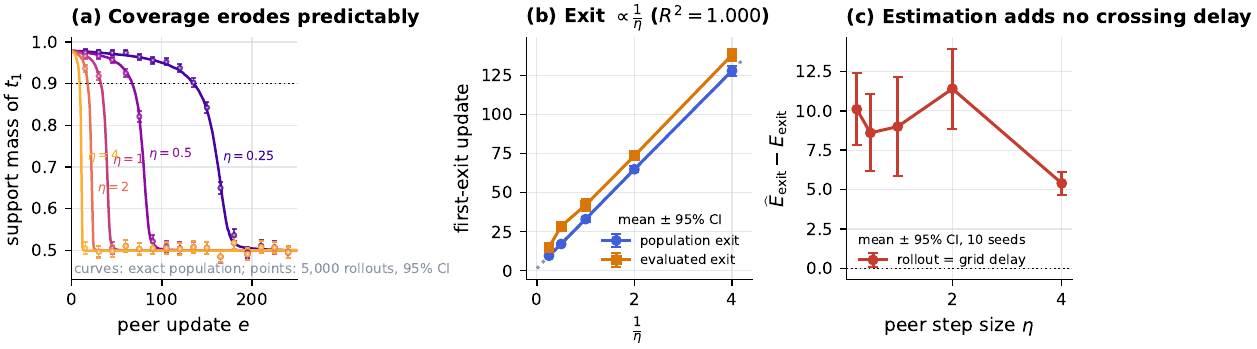}
\caption{\textbf{E4.} \captiona~Exact population coverage after every update (curves) and $5{,}000$-rollout estimates with approximate 95\% conditional-binomial intervals every $15$ updates (points), first seed. \captionb~Population and finite-evaluation first exits against $\frac{1}{\eta}$ (mean and 95\% CI across 10 training seeds); the population relation has $R^2=0.999986$. \captionc~The rollout-estimated and exact-grid delays coincide in all runs, so one curve displays their common mean and 95\% CI.}
\label{fig:pg}
\end{figure}

\paragraph{E5: membership vs.\ executability.} Lowering the peer's hold-probability $\beta$, the prototypes $t_1,t_2$ remain coverage-1 in the $\delta{=}0$ family for all $\beta>0$ (membership intact), yet a fixed option chain calibrated at $\beta{=}1$ drops from success probability $1.0$ to $0.10$ as $\beta{\to}0.05$ (\Cref{prop:exec-drift}).
\begin{figure}[h]
\centering
\includegraphics[width=0.42\linewidth]{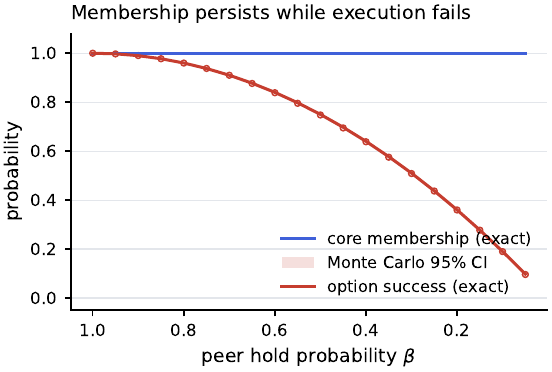}
\caption{\textbf{E5.} Exact membership of $t_1,t_2$ is retained for all $\beta>0$ while a fixed option chain's exact success collapses---executability drift without membership drift. Open points and the 95\% Monte-Carlo band use $40{,}000$ trials.}
\label{fig:exec}
\end{figure}

\paragraph{Effect of the core.} \Cref{fig:coreeffect}\captiona\ measures the fraction of the length-$H$ string space surviving the soundness constraint; since the frontier string has length $\Theta(L)$ it shrinks exponentially in $L$, never excluding an optimum under the full-support probe (\Cref{prop:sound}). \Cref{fig:coreeffect}\captionb\ plots that frontier-string length, and \Cref{fig:coreeffect}\captionc\ contrasts a core-guided option chain (100\% success) with a core-blind random-admissible controller ($\sim$17--28\%).
\begin{figure}[h]
\centering
\includegraphics[width=0.58\linewidth]{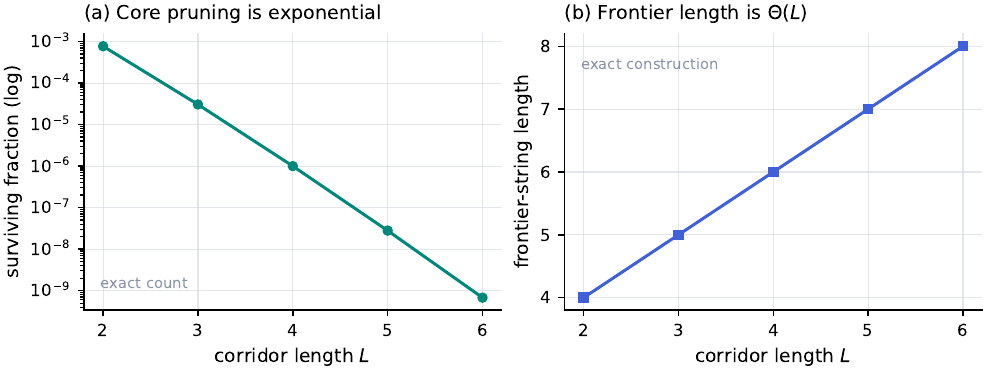}\hfill
\raisebox{0.2em}{\includegraphics[width=0.36\linewidth]{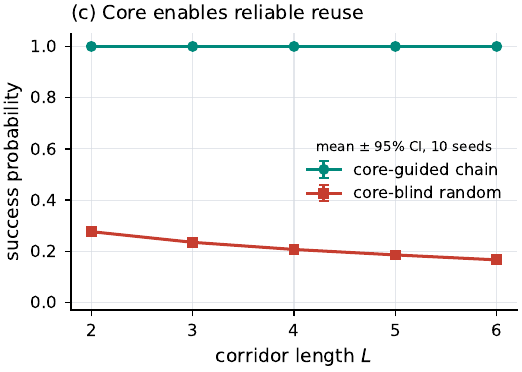}}
\caption{\textbf{Effect of the core.} \captiona~The exact soundness count leaves an exponentially shrinking fraction of trajectory strings as $L$ grows. \captionb~The frontier here is a \emph{single} exact string, so $|\Core^0_\phi(\mu)|=1$; its \emph{length} grows as $\Theta(L)$. \captionc~Core-guided ($100\%$) vs.\ core-blind ($\sim$17--28\%), with means and 95\% CIs across 10 seeds.}
\label{fig:coreeffect}
\end{figure}

\paragraph{Core-computation cost.} The exact multi-sequence LCS cost (DP cells) grows as $O(H^k)$ in the horizon for each number of successes $k$, with the fitted exponent tracking $k$ (\Cref{sec:complexity}).
\begin{figure}[h]
\centering
\includegraphics[width=0.92\linewidth]{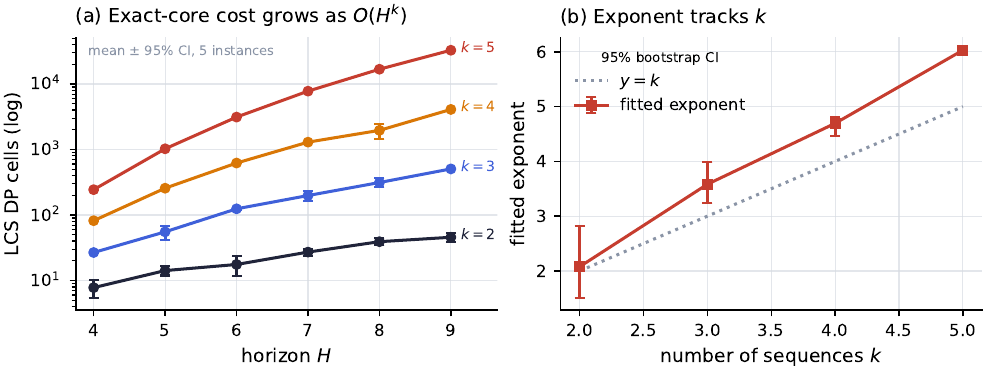}
\caption{\textbf{Core-computation cost.} \captiona~Exact multi-string LCS cost (DP cells) grows as $O(H^k)$ (log scale); bars are 95\% CIs over five random instances per cell. \captionb~The fitted exponent tracks $k$, with 95\% instance-bootstrap intervals, verifying the product-index DP profile of \Cref{sec:complexity} for the support core ($\delta{=}0$).}
\label{fig:corecost}
\end{figure}

\paragraph{Scaling: the effect grows with $N$ and $L$.} Define core fragility as $\frac{10^3}{E_{\mathrm{exit}}}$, where $E_{\mathrm{exit}}$ is the first update at which no length-1 top-route prototype remains in the estimated $\delta{=}0.05$ family. Both fragility and aggregate boundary drift increase monotonically with $L$ and $N$ (\Cref{fig:surface}). The slices in \Cref{fig:survival} show survival $\propto \frac{1}{N}$ ($R^2=0.999$), as predicted by the scalar coupling $q=1-(1-\alpha)^N$ of SRC-$L$-$N$. They also show that full-core survival decreases with $L$.

The bound is never violated at any tested $(L,N)$. \Cref{fig:scalebound} reports the minimum slack over drift levels where the guarantee is non-vacuous. It ranges from $0.009$ at $(L,N)=(2,1)$ to $0.289$ at $(8,8)$. The guarantee becomes more conservative as drift accumulates and becomes vacuous once $\varepsilon_e\ge(1-\delta)p_0$.
\begin{figure}[h]
\centering
\includegraphics[width=0.92\linewidth]{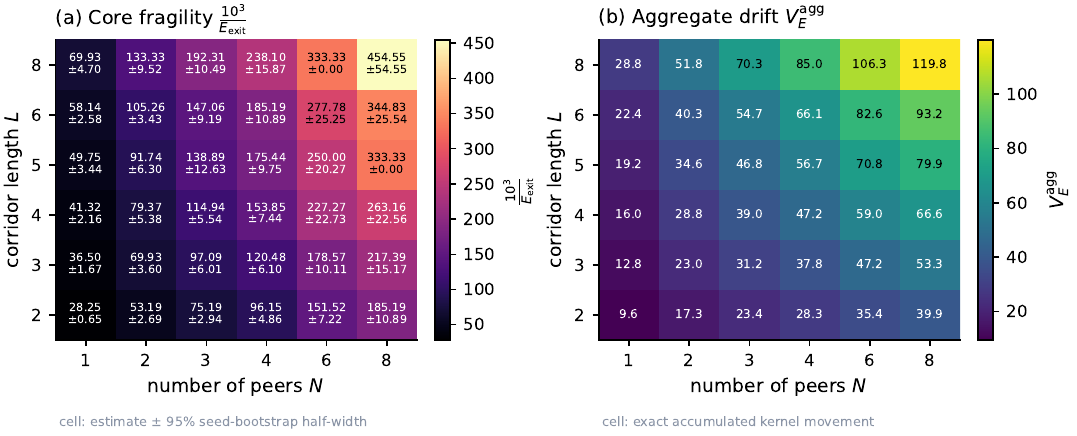}
\caption{\textbf{Scaling grid ($L\times N$).} \captiona~Core fragility $\frac{10^3}{E_{\mathrm{exit}}}$ grows in both $N$ and $L$; cells report the estimate and 95\% seed-bootstrap half-width over 10 seeds. \captionb~Exact aggregate boundary drift grows in both.}
\label{fig:surface}
\end{figure}
\begin{figure}[h]
\centering
\includegraphics[width=0.92\linewidth]{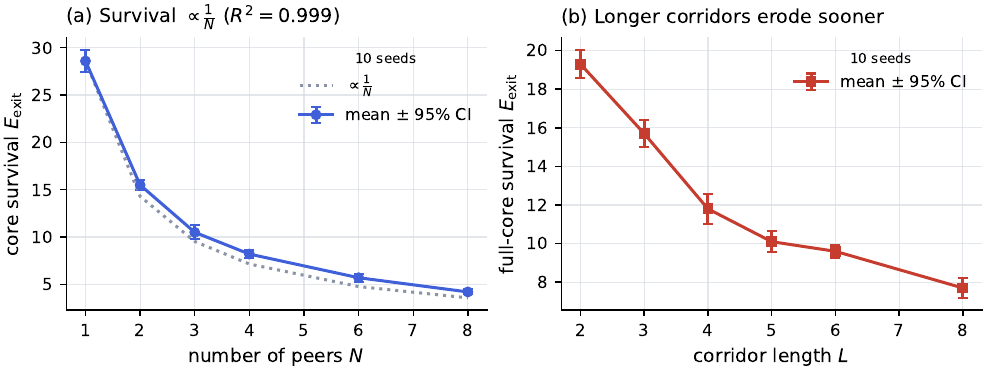}
\caption{\textbf{Scaling slices (10 seeds).} Points and bars are means and 95\% CIs. \captiona~$E_{\mathrm{exit}}\propto \frac{1}{N}$ ($R^2=0.999$). \captionb~Full-core survival decreases as the corridor lengthens.}
\label{fig:survival}
\end{figure}
\begin{figure}[h]
\centering
\includegraphics[width=0.46\linewidth]{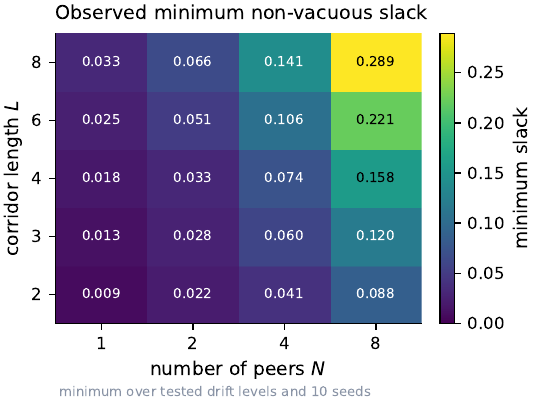}
\caption{\textbf{Bound validity across the grid.} Observed minimum slack (empirical coverage minus the \Cref{thm:stab} guarantee) over tested non-vacuous drift levels and 10 seeds. There are no violations; slack grows with $L$ and $N$. Every displayed cell has at least one tested non-vacuous drift level.}
\label{fig:scalebound}
\end{figure}

\subsection{Direct quantitative validations of theorem-level predictions (E7--E9)}\label{app:theorytests}
E7 tests the predicted sampling rate, while E8 attains the factor one at the trajectory-law level and E9 shows that the kernel-parameterized coefficient is unimprovable uniformly over horizons.

\paragraph{E7: the $\frac{1}{p_0\gamma^2}$ rate.} The plug-in estimator uses only two counts: the number of successes and the number of successes containing $u$. Under i.i.d.\ rollouts, $N_{\suc}\sim\mathrm{Bin}(N,p_0)$, and the occurrence count conditional on $N_{\suc}$ follows $\mathrm{Bin}(N_{\suc},c_\mu(u))$. SRC rollouts verify this reduction (means $0.7128$ vs $0.7132$, standard deviations $0.0092$ vs $0.0094$ at $N{=}4000$).

We then sweep $p_0\in[0.05,0.8]$ and $\gamma\in[0.03,0.08]$. Varying the probe alone moves $p_0$ by under $10\%$, so the direct sweep is needed to exercise the $\frac{1}{p_0}$ factor. For each cell, $N^\star$ is the smallest sample size at which the plug-in rule misclassifies with probability at most $5\%$ over $4000$ repetitions. Across $15$ cells, $N^\star$ spans $126$ to $13{,}934$, while $N^\star p_0\gamma^2$ remains in $[0.574,0.692]$, a spread of $1.21\times$ (\Cref{fig:samplecomplexity}). Direct simulation of the exact sampling law therefore confirms the predicted rate and its stable constant over the tested grid.
\begin{figure}[h]
\centering
\includegraphics[width=0.92\linewidth]{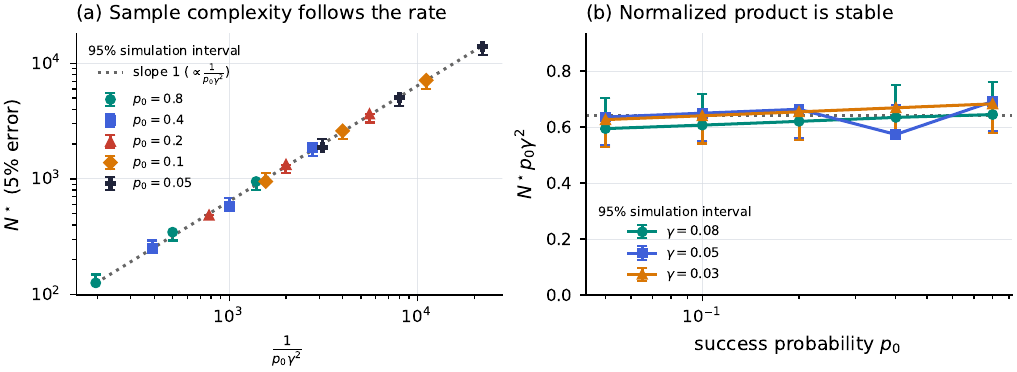}
\caption{\textbf{E7.} \captiona~$N^\star$ against $\frac{1}{p_0\gamma^2}$ on log axes, with a slope-one reference. \captionb~$N^\star p_0\gamma^2$ across all $15$ cells, constant to within $1.21\times$. Bars bracket the 95\% Wilson uncertainty in the simulated 5\% error crossing using $4{,}000$ repetitions per tested $N$.}
\label{fig:samplecomplexity}
\end{figure}

\paragraph{E8: worst-case tightness of \Cref{thm:stab}.} Tightness is measured at the \emph{trajectory-law} level using the actual $\TV(\mu_\alpha,\mu_0)$ over the abstract outcome classes rather than the kernel bound of \Cref{lem:sim}. SRC moves success mass precisely from ``$u$ occurs'' to ``$u$ absent'' while leaving total success mass nearly fixed, which is the equality case. The realized coverage drop never exceeds the guaranteed maximum $\frac{\TV}{p_0}$ at any of the $14$ drift levels, and the ratio reaches $0.996$ at the smallest drift, decaying to $\approx0.66$ as the drift grows (\Cref{fig:tightness}). The factor one is attained in the vanishing-drift limit.
\begin{figure}[h]
\centering
\includegraphics[width=0.92\linewidth]{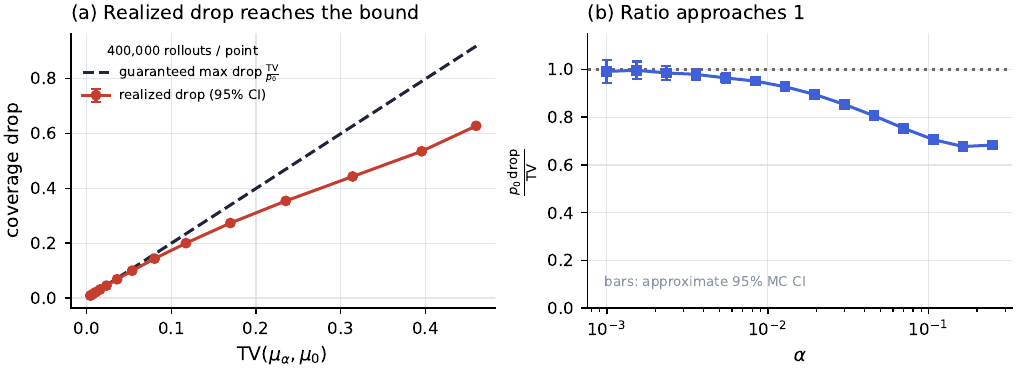}
\caption{\textbf{E8.} \captiona~Realized coverage drop against the guaranteed maximum $\frac{\TV}{p_0}$. \captionb~Their ratio approaches $1$ as drift vanishes and never exceeds it. Bars are approximate 95\% Monte-Carlo intervals from $400{,}000$ rollouts per point.}
\label{fig:tightness}
\end{figure}

\paragraph{E9: tightness in the kernel parameterization.} E8 establishes tightness at the trajectory-law level, where $\varepsilon_e$ is the actual $\TV(\mu_e,\mu_0)$. The kernel-parameterized guarantee, including the telescoping in \Cref{lem:sim}, is tested by an explicit depth-indexed family. In the \emph{cascade} $\mathrm{Casc}(n)$ the focal agent has a single action and positions $t_0,\dots,t_n$ (top), $b_1,\dots,b_n$ (bypass), and $g$. At $t_0,\dots,t_{n-1}$ the peer either holds (continue on top) or, with probability $\alpha$, diverts to the bypass; $t_n$ then enters $g$. Both routes reach $g$, so $p_0=1$, and diverted mass lands on successes lacking $u=(t_1,\dots,t_n)$. A top success has labels $(t_0,\dots,t_n,g_\Sigma)$, hence $H=n+2$, $\norm{P_\alpha-P_0}_{1,\infty}=2\alpha$, and the theorem gives $\varepsilon_\alpha=(n+1)\alpha$. The realized drop is exactly $1-(1-\alpha)^n$, so Bernoulli's inequality keeps it below the guarantee and
\[
\lim_{\alpha\downarrow0}\frac{1-(1-\alpha)^n}{(n+1)\alpha}=\frac{n}{n+1}.
\]
For fixed $n=11$ this limit is $\frac{11}{12}$; across the family it tends to one as $n\to\infty$, reaching $\frac{255}{256}=0.9961$ at $n=255$ (\Cref{fig:kerneltight}). Thus no horizon-uniform coefficient smaller than one can replace the kernel factor in \Cref{thm:stab}.
\begin{figure}[h]
\centering
\includegraphics[width=0.92\linewidth]{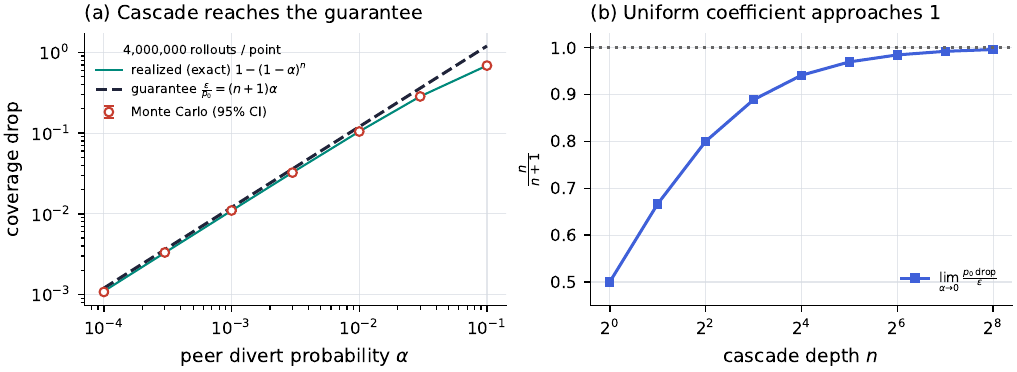}
\caption{\textbf{E9.} \captiona~For $n=11$, realized coverage drop (exact and Monte Carlo with 95\% binomial intervals from $4{,}000{,}000$ rollouts) against the corrected kernel guarantee $\frac{\varepsilon}{p_0}=(n+1)\alpha$. \captionb~The infinitesimal-drift ratio is $\frac{n}{n+1}$ and approaches one with cascade depth, proving that the horizon-uniform coefficient is sharp.}
\label{fig:kerneltight}
\end{figure}

\clearpage
\subsection{Recorded configurations and comparison inputs}\label{app:parameters}
The tables collect the frozen configuration values used for E10--E12. Each independent unit is a complete stream or peer/focal pairing; checkpoint trajectories remain intact under resampling.

\begingroup
\small
\renewcommand{\arraystretch}{1.12}
\begin{longtable}{@{}p{0.27\linewidth}@{\hspace{8pt}}p{0.67\linewidth}@{}}
\toprule
\multicolumn{2}{@{}l}{\textbf{E10: continual control}}\\
\midrule
Plant and reference set & $a=1.08$, $R=0.75$, state limit $2.5$, action limit $0.45$, process-noise standard deviation $0.006$.\\
Controllers and probe & Nominal gain $0.50$, backup gain $2.00$, probing amplitude $0.035$, robust disturbance radius $w_{\max}=0.080$.\\
Costs & Action-cost coefficient $0.015$; backup cost $0.010$ per step.\\
Gain schedule & High gain $1.0$; low gain in $[0.045,0.075]$; ramp lengths $300$--$500$, low plateaus $650$--$950$, stable phases $650$--$1050$ steps; initial stable phase $700$--$1100$.\\
RLS & Initial gain estimate $1.0$, initial covariance $8.0$, forgetting factor $0.985$, estimate clipped to $[0,1.25]$.\\
Monitoring & Every $20$ steps; reward window $100$; healthy/failure reward thresholds $0.90/0.70$; prospective horizon $240$ steps.\\
Switching hysteresis & Coverage enter/exit $0.95/0.985$; return enter/exit $0.82/0.94$; minimum backup duration $80$ steps.\\
Inference & $64$ confirmation streams; $10{,}000$ paired bootstrap resamples, seed $884219$.\\
\midrule
\multicolumn{2}{@{}l}{\textbf{E11: LBF coordination fork}}\\
\midrule
Peer parameterization & Feed-forward logistic route policy; initial route bias $-4.6$, cue weight $1.5$, style logit $-1.0$; initialization standard deviation $0.04$.\\
Peer updates & Sampled REINFORCE with batch size $512$, route learning rate $0.30$, style learning rate $0.35$, $32$ updates per checkpoint; initial anchoring for $16$ updates.\\
Coordination mechanism & Neutral wait-then-follow probe; probe and robust-policy success parameters $0.995$ and $0.970$; mismatch-recovery probability $0.03$; style latency cost $0.015$.\\
Discovery & Four separate pairings, $2000$ episodes each; minimum candidate coverage $0.95$, maximum length $4$, peer-action symbols.\\
Evaluation & Eight development pairings, $72$ checkpoints; separate batches of $500$ probe, $500$ focal-monitoring, and $500$ performance episodes per checkpoint.\\
Thresholds & Nominal confidence $0.95$; coverage membership/recovery $0.90/0.94$; performance healthy/failure $0.90/0.70$; prediction horizon $3$ checkpoints.\\
Inference & $2000$ complete-pairing bootstrap resamples, seed $20260831$; post-hoc switch from absolute risk to erosion on these eight pairings; window matched to the already-frozen three-checkpoint failure horizon.\\
\midrule
\multicolumn{2}{@{}l}{\textbf{E12: continual cue-MNIST}}\\
\midrule
Stream and splits & $300$ checkpoints, four events; $55{,}000$ training examples; disjoint core/shift/performance probes of $2000/500/1000$ examples.\\
Network and optimizer & $784\to64$ tanh$\to10$ logits; $50{,}890$ parameters. Adam $(0.9,0.999)$, $\epsilon=10^{-8}$; weight-matrix L2 regularization.\\
Pretraining & $3500$ supervised updates; batches of $64$ base images with two cue views each; learning rate $0.0015$, L2 coefficient $10^{-5}$.\\
Continual updates & $299$ intervals of $12$ updates, batch size $128$; learning rate $0.001$, L2 coefficient $0.001$; checkpoint zero is pretrained.\\
Cue and image parameters & $10$ spatial $3\times3$ cues, intensity $1.0$; $q_{\mathrm{ref}}(y)=(y+1)\bmod10$; reference/harmful contrast $0.30$, nuisance contrast $0.27$. Harmful cue-label mixture probability $0.98$; benign cues uniform.\\
Nested cue sets & $U_1=\{0,1\}$, $U_2=\{0,\ldots,4\}$, $U_3=\{0,\ldots,9\}$; the primary coverage requires correctness under all ten cues.\\
Monitoring & Three-checkpoint erosion; healthy/failure/recovery accuracy $0.92/0.88/0.92$; future horizon $7$ checkpoints.\\
Replay & Coverage enter/exit $0.90/0.94$; replay fraction $0.50$; intervention cost $0.002$ per checkpoint.\\
Inference & $64$ confirmation streams; $10{,}000$ paired bootstrap resamples, seed $20261211$.\\
\bottomrule
\caption{Recorded E10--E12 configuration values.}\label{tab:recorded-parameters}
\end{longtable}
\endgroup

\paragraph{Information used by the comparisons.}
In E10, the structural monitor combines observed transitions and an RLS estimate of the unknown actuator gain with the specified plant coefficient $a$, nominal feedback gain, reference interval, and disturbance set. The return monitors use reward history, the state-magnitude baseline uses the current state, and transition-surprise CUSUM uses transition residuals. The reported comparison therefore concerns these task-specific monitors with their stated inputs; it does not isolate the contribution of structural coverage from model knowledge. The operational monitor receives neither the true actuator gain nor future drift. The oracle substitutes the true gain in the same viability formula.

In E12, the structural monitor queries counterfactual cues on the core probe, cue association uses the separate shift-monitor split, and the accuracy trigger uses reference-cue accuracy on the same core probe; held-out performance is evaluated separately. Replay-trigger comparisons share the registered replay procedure and differ in intervention timing. These comparisons evaluate trigger choice within a fixed learning setup. Replay-based continual-learning algorithms such as CLEAR \citep{rolnick2019clear}, Bayesian change-point inference \citep{adams2007changepoint}, and partner-population training such as fictitious co-play \citep{strouse2021fcp} address complementary algorithmic questions. They are literature context, not evaluated baselines in these studies.

\paragraph{Which theoretical quantities are tested.}
SRC evaluates conditional coverage, trajectory drift, and first exits directly, including rare-success amplification. E10 instantiates a length-one coverage functional with success mass one; E12 uses a length-one cue-robust event conditional on reference correctness, with success mass $\frac{|B_t|}{2000}$. Their stream experiments test prediction and intervention under exogenous drift. E11 traces peer-induced coverage erosion in a controlled coordination fork. Monitoring these fixed candidates uses event evaluation rather than recomputation of the maximal subsequence frontier.

\paragraph{Artifact scope.}
The supplied source runs the SRC studies and regenerates E10--E12 statistics figures from frozen aggregate and per-unit records. The cue-MNIST training and analysis source is also supplied, together with runnable configurations and source-verification checks. The compact E10 and E11 bundles contain protocols, configurations, and records; their original training drivers and large per-step arrays are not included. E12 implementation details are specified in \Cref{app:cue-mnist}.

\subsection{E10: registered single-agent continual-control confirmation}\label{app:continual-control}
E10 tests a downstream consequence of the focal continual-RL setting represented by \Cref{prop:reduction}: degradation of reusable structure should warn before performance fails, and acting on that warning should preserve performance (\Cref{fig:main-experiments}, top; \Cref{fig:task-schematics}, left). Each seed runs uninterrupted for $18{,}000$ steps:
\[
x_{t+1}=a x_t+b_tu_t+w_t,
\]
where actuator gain $b_t$ undergoes five seed-randomized continuous high-to-low-to-high drift--recovery events. Physical state, recursive-least-squares (RLS) gain estimate, rolling monitors, controller state, and intervention state are never reset. The action is clipped; a nominal controller is inexpensive but loses stability at low gain, while a costed backup remains stable throughout the registered range. Each treatment receives the same drift schedule, process noise, and probing sequence within a seed but evolves its own unbroken state and RLS estimate.

On the monitored reference interval $X=[-R,R]$, clipping is inactive for the nominal controller even at maximal probing amplitude: the registered constants satisfy $k_{\mathrm{nom}}R+a_{\mathrm{probe}}=0.410<0.45=u_{\max}$.  The disturbance radius also satisfies $0\le w_{\max}=0.080<R=0.75$. Action saturation therefore does not alter the nominal feedback law on $X$, and the viability set below is nonempty whenever its displayed radius is positive; it is the one-step viability core of the monitor's estimated linear model.

\paragraph{Structural quantity and prospective target.}
For the fixed reference interval $X=[-R,R]$ and disturbance set $W=[-w_{\max},w_{\max}]$, define the estimated one-step viability core of the nominal feedback controller by
\[
\widehat K_t:=\left\{x\in X:
(a-\widehat b_t k_{\mathrm{nom}})x+w\in X
\ \text{for every }w\in W\right\}.
\]
Its normalized coverage is
\[
\widehat c_t:=\frac{\operatorname{Leb}(\widehat K_t)}{\operatorname{Leb}(X)}
=
\begin{cases}
1,&a-\widehat b_t k_{\mathrm{nom}}=0,\\
\min\!\left\{1,\frac{R-w_{\max}}{R\,|a-\widehat b_t k_{\mathrm{nom}}|}\right\},&\text{otherwise}.
\end{cases}
\]
Thus $\widehat c_t$ is exactly the fraction of $X$ that remains viable under every registered disturbance; operational core risk is $1-\widehat c_t$. The construction below represents this quantity exactly as the coverage of a length-one candidate. The latent $b_t$ appears only in an oracle diagnostic. At a checkpoint with healthy current rolling reward, the positive prospective label is a new downward crossing of the registered failure threshold within the next $240$ steps. Existing failure plateaus and incomplete tail windows are censored, and recovery re-enters the risk set.

\paragraph{Exact correspondence to conditional-event coverage.}
Set $q_t=a-\widehat b_tk_{\mathrm{nom}}$, draw $X_0\sim\operatorname{Unif}(X)$, and let $\mu_t$ be the law of
\[
Z_t=(X_0,q_tX_0-w_{\max},q_tX_0+w_{\max}).
\]
All $\mu_t$ live on the same three-coordinate space.  Take the success event to be that whole space, so $p_t=1$, and define the fixed structural event
$A_{\mathrm{ctrl}}:=\{Z:Z_2,Z_3\in X\}$.  Because an affine image of the disturbance interval attains its extrema at $\pm w_{\max}$,
\[
X_0\in\widehat K_t
\iff |q_tX_0|+w_{\max}\le R
\iff Z_t\in A_{\mathrm{ctrl}}.
\]
Therefore
$\mu_t(A_{\mathrm{ctrl}}\mid\suc)=\Pr[X_0\in\widehat K_t]
=\frac{\operatorname{Leb}(\widehat K_t)}{\operatorname{Leb}(X)}=\widehat c_t$.
Equivalently, form a one-step probe MDP with initial-state distribution $\mu_t$. Its only action is a dummy action, its abstraction emits \textsf{SAFE} exactly on $A_{\mathrm{ctrl}}$ (and \textsf{OTHER} otherwise), and every terminal state is successful. Then $\widehat c_t$ is exactly the \Cref{def:core} coverage of the length-one candidate $(\textsf{SAFE})$. E10 uses this exact coverage as a monitor and tests its downstream predictive and control consequence; it does not estimate a peer-policy drift budget.

\paragraph{Separation, inference, and interventions.}
Thirty-two full development streams (seeds $5000$--$5031$) selected return slope as the strongest of six simple baselines. The plant, drift distribution, estimator, structural quantity, prediction target, thresholds, controllers, and intervention costs were then frozen before running 64 confirmation streams (seeds $0$--$63$). The two registered co-primary contrasts are structural-risk minus return-slope average precision and certificate-trigger minus return-trigger whole-stream mean reward. Their 95\% percentile-bootstrap intervals use $10{,}000$ resamples of the 64 complete streams. Checkpoints, the 320 drift incidents, and time steps are not treated as independent replicates. Besides the return trigger, controls are never switching, an intervention-count-matched random cyclic shift of the certificate schedule, and a latent-gain oracle.

\paragraph{Confirmation results.}
Core risk has average precision $0.5922$, versus $0.3551$ for the frozen return-slope baseline; the paired advantage is $0.2370$ (95\% CI $[0.2233,0.2509]$). Mean warning lead is $83.24$ steps for core risk and $6.07$ for the return signal. Core-triggered control has mean reward $0.99444$, compared with $0.45778$ for the return trigger, $0.46035$ for never switching, $0.53558$ for random-matched intervention, and $0.99447$ for the oracle. The paired core-triggered advantages are respectively $0.53666$ ($[0.53058,0.54281]$), $0.53409$ ($[0.52790,0.54036]$), and $0.45886$ ($[0.43605,0.48086]$). Both registered co-primary lower endpoints exceed zero, so the pre-specified intersection--union rule passes. The core-triggered arm records no failure steps across the 64 streams.

\paragraph{Protocol and reproducibility.}
The experiment validates the downstream core mechanism in a genuinely continual stream: erosion of a monitored structural core anticipates failure, and acting on that erosion preserves performance. The imported compact artifact at \texttt{code/continual\_control\_toy\_v1\_confirmation\_stats/} contains both frozen protocols and configurations, aggregate summaries, every per-stream statistic, integrity decisions, and a SHA-256 manifest; all files pass that manifest. The large per-step arrays remain on the originating cluster. Figure regeneration uses \texttt{code/plot\_continual\_control\_confirmation.py}; the illustrative task image and its generation prompt are stored beside the figures.

\subsection{E11: exploratory peer-induced core erosion in Level-Based Foraging}\label{app:lbf-exploratory}
E11 is an exploratory reanalysis of eight development pairings from a controlled Level-Based Foraging (LBF) coordination-fork experiment (\Cref{fig:main-experiments}, bottom; \Cref{fig:task-schematics}, right). The joint task and rewards remain stationary while a feed-forward logistic route peer is updated by sampled REINFORCE through two right-learning and left-recovery cycles. A fixed neutral focal probe waits for the peer's route and follows it, whereas a separately frozen old-left focal emits its original cue. Their monitoring and held-out performance rollouts are disjoint.

The independently discovered peer sequence is \textsf{ENTER-LEFT-GATE}, \textsf{LEFT-READY}, \textsf{LOAD-LEFT}. Four discovery pairings use $2{,}000$ episodes each, and all discovery and baseline-coverage gates pass. Each learned-peer pairing contains 72 peer checkpoints, with 500 probe, 500 focal-monitoring, and 500 performance episodes per checkpoint. An additional style-learning block changes action and occupancy statistics without changing the route convention, providing a nuisance-drift control.

\paragraph{Registered result and exploratory reanalysis.}
Let $L_t$ denote the Wilson lower confidence bound on candidate coverage. The registered V3 primary predictor was absolute structural risk, $1-L_t$. It obtained mean AP $0.5260$, below the one-checkpoint monitoring-return slope baseline at $0.7428$; the registered paired contrast was $-0.2168$ with a 95\% pairing-bootstrap interval $[-0.2884,-0.1520]$, so the registered prediction gate failed.

Here the correspondence to \Cref{def:core} is direct: at peer checkpoint $t$, $\mu_t$ is the neutral focal's fixed-probe trajectory law, $\suc$ is successful coordination, and $u=(\textsf{ENTER-LEFT-GATE},\textsf{LEFT-READY},\textsf{LOAD-LEFT})$.  The monitored population target is $c_{\mu_t}(u)$ and $L_t$ is its finite-rollout lower confidence bound.  Peer learning changes $\mu_t$ while the focal probe, abstraction, success event, and candidate remain fixed.

After examining the same eight pairings, the reanalysis replaced the registered absolute-risk predictor with active erosion. Its three-checkpoint window matches the already-frozen prospective failure horizon:
\[
s_t^{(3)}=\max\!\left\{0,L_{t-3}-L_t\right\}.
\]
The prospective label is a new downward crossing of old-focal held-out success below $0.70$ within the next three checkpoints, restricted to currently eligible checkpoints. The erosion score measures a windowed coverage change; \Cref{prop:window-erosion} gives its directional-progress and estimation-error decomposition. It attains mean AP $0.9712$, versus $0.7428$ for return slope, and exceeds that baseline in all eight pairings. The paired advantage is $0.2284$ (pairing-bootstrap 95\% CI $[0.1707,0.2773]$). Windows of one through five checkpoints give mean AP $0.9095,0.9653,0.9712,0.9626,0.9252$, respectively, all above return slope. The neighboring-window results are sensitivity diagnostics. The pairing-bootstrap intervals condition on the erosion predictor and do not account for the post-hoc change of predictor. Because erosion was evaluated on the same development pairings that informed this change, these results are exploratory and motivate confirmation on new pairings.

\paragraph{Protocol and reproducibility.}
The exploratory pattern contains the full downstream chain in a multi-agent environment: peer learning changes the induced focal problem, a fixed-probe trajectory core erodes, and old-partner performance later fails. The imported artifact at \texttt{code/lbf\_v3\_core\_erosion\_exploratory\_stats/} contains the frozen V3 protocol and configuration, all eight raw trajectory summaries, the registered result, the post-hoc aggregate and per-pairing reanalysis statistics, and SHA-256 checksums. Figure regeneration uses \texttt{code/plot\_lbf\_core\_erosion\_exploratory.py}; the illustrative task image and its complete generation/edit prompt set are stored beside the figures. The reported APs and pairing-bootstrap intervals were independently recomputed from the supplied summaries.

\subsection{E12: registered continual cue-MNIST confirmation}\label{app:cue-mnist}
E12 trains a persistent classifier through four randomized shortcut-learning and recovery cycles. The reported confirmation comprises 64 supervised streams (seeds $8200$--$8263$); model and optimizer state persist across phases.

\paragraph{Images, cues, and network.}
Each normalized $28\times28$ MNIST image is multiplied by contrast $0.30$, then a single $3\times3$ patch of intensity one is overwritten at one of ten cue locations. Cue identities $0$--$4$ use zero-based top-left coordinates $(1,1),(1,7),(1,13),(1,19),(1,25)$; identities $5$--$9$ use the same columns in row $25$. Nuisance phases reduce contrast to $0.27$. The input remains a flattened $784$-pixel image. The classifier has architecture $784\to64\ (\tanh)\to10$ logits, with $50{,}890$ trainable parameters and argmax evaluation. Entries of $W_1$ and $W_2$ are initialized independently as $\mathcal N(0,\frac{2}{784})$ and $\mathcal N(0,\frac{2}{64})$, respectively; biases are zero and computations use float32. During harmful phases, the cue equals the digit label with mixture probability $0.98$ and is otherwise uniform; hence cue--label agreement has probability $0.982$. Benign phases use independent uniform cues. Digit labels never change.

\paragraph{Training and disjoint data.}
The first $55{,}000$ MNIST training images supply learning batches. A seed-specific permutation of the remaining $5000$ provides disjoint fixed core and shift probes of $2000$ and $500$ images; a separate fixed $1000$-image subset of the official test set measures performance. Supervised pretraining uses $3500$ updates with batches of $128$ rendered images: $64$ base images are drawn with replacement, each receiving two distinct cues $j$ and $(j+1)\bmod10$. Pretraining minimizes mean cross-entropy at learning rate $0.0015$ with weight-matrix L2 coefficient $10^{-5}$. Continual learning uses the same loss, batch size $128$, learning rate $0.001$, and L2 coefficient $0.001$; biases are unregularized. Adam uses $(\beta_1,\beta_2)=(0.9,0.999)$, $\epsilon=10^{-8}$, and bias correction. Each treatment copies both pretrained parameters and Adam state. Checkpoint zero records the pretrained model; the remaining $299$ intervals contain $12$ updates each ($3588$ continual updates).

\paragraph{Success-conditioned cue coverage.}
Write $f_t(x,j)$ for the predicted label after rendering image $x$ with cue $j$ at reference contrast, and set $q_{\mathrm{ref}}(y)=(y+1)\bmod10$. Reference-cue accuracy therefore uses a label-dependent cue, not an unmodified image. The nested probe sets are $U_1=\{0,1\}$, $U_2=\{0,\ldots,4\}$, and $U_3=\{0,\ldots,9\}$. At each checkpoint, recompute
\[
B_t:=\{i:f_t(x_i,q_{\mathrm{ref}}(y_i))=y_i\},\qquad
\widehat c_t:=\frac{1}{|B_t|}\sum_{i\in B_t}
\Ind\!\left[\forall j\in U_3,\ f_t(x_i,j)=y_i\right].
\]
Thus $\widehat c_t$ measures correctness under \emph{all ten} cues among currently reference-correct examples. For $n_t=|B_t|>0$, the operational score is the lower endpoint of a nominal two-sided $95\%$ Wilson interval:
\[
C_t=\frac{\widehat c_t+\frac{z^2}{2n_t}
-z\sqrt{\frac{\widehat c_t(1-\widehat c_t)}{n_t}+\frac{z^2}{4n_t^2}}}
{1+\frac{z^2}{n_t}},
\qquad z=1.959963984540054.
\]
If $B_t=\varnothing$, conditional coverage is undefined (NaN in the implementation) and the operational fallback is $C_t=0$. This case never occurs in the $19{,}200$ recorded untreated confirmation checkpoints: $|B_t|\ge1040$. Since $q_{\mathrm{ref}}(y)\in U_3$, unconditional compatible mass is $\frac{|B_t|}{2000}\widehat c_t$. Each checkpoint evaluates $2000$ reference renderings and $20{,}000$ cue variants, representing $2000$ base examples; repeated renderings are not independent observations. The fixed probe supplies repeated monitoring, while inference on prediction and intervention resamples complete streams.

\paragraph{Exact correspondence to conditional-event coverage.}
Draw an index $I$ uniformly from the fixed $2{,}000$-example probe and define
\[
Z_t=\left(y_I,f_t(x_I,q_{\mathrm{ref}}(y_I)),(f_t(x_I,c))_{c\in U_3}\right).
\]
Let $\mu_t$ be the empirical law of $Z_t$ on this common coordinate space.  Define the fixed events
\[
\suc_{\mathrm{cue}}:=\{Z_{\mathrm{ref}}=Z_y\},
\qquad
A_{\mathrm{cue}}:=\suc_{\mathrm{cue}}\cap\{Z_c=Z_y\ \text{for every }c\in U_3\}.
\]
Then $\frac{|B_t|}{2000}=\mu_t(\suc_{\mathrm{cue}})$ and the displayed estimator satisfies
\[
\widehat c_t=\mu_t(A_{\mathrm{cue}}\mid\suc_{\mathrm{cue}}).
\]
Form a one-step probe MDP with initial-state distribution $\mu_t$. Its only action is a dummy action, its goal set is $\suc_{\mathrm{cue}}$, and its abstraction emits \textsf{CUE-ROBUST} exactly on $A_{\mathrm{cue}}$ (and \textsf{OTHER} otherwise). Its \Cref{def:core} coverage for the length-one candidate $(\textsf{CUE-ROBUST})$ is exactly $\widehat c_t$. The network update changes $\mu_t$ while the probe indices, cues, coordinate events, success set, abstraction, and candidate remain fixed. E12 tests whether erosion of this coverage predicts failure and whether intervention triggered by its lower bound preserves performance.

\paragraph{Prospective prediction and intervention.}
The frozen prediction score is $s_t=\max\{0,C_{t-3}-C_t\}$, with $s_t=0$ for $t<3$, as analyzed in \Cref{prop:window-erosion}. Its target is a reference-cue accuracy downcrossing below $0.88$ within seven checkpoints, evaluated from healthy checkpoints at or above $0.92$; the frozen comparator is cue--label association. Replay instead uses the \emph{level} $C_t$: it activates below $0.90$ and deactivates above $0.94$. The accuracy comparator uses reference-cue accuracy on the same core probe, with enter/exit thresholds $0.88/0.92$. The decision at checkpoint $t$ controls the next training interval. Active replay replaces the first $64$ renderings of each $128$-image batch with independently random-cued renderings of those same base images at reference contrast. On the separate performance split, let $A_{\mathrm{ref},t}$ and $A_{\mathrm{current},t}$ denote argmax accuracies under reference and current cues. The reported balanced reward is the whole-stream mean of
\[
R_t=\frac{A_{\mathrm{ref},t}+A_{\mathrm{current},t}}{2}
-0.002\,\Ind\{\text{replay active in the interval ending at }t\}.
\]
The configuration, score, comparator, thresholds, seed list, treatment rules, and analysis were frozen before confirmation. Paired $95\%$ intervals use $10{,}000$ nonparametric bootstrap resamples of the 64 complete streams. Checkpoints, incidents, and images remain within-stream observations.

\paragraph{Confirmation results.}
Core erosion has mean AP $0.3842$, versus $0.3450$ for cue--label association, for a paired advantage of $0.0392$ (95\% CI $[0.0224,0.0555]$). Mean structural warning lead is $4.68$ checkpoints. Core-triggered replay has mean balanced reward $0.94995$, versus $0.93824$ for the accuracy trigger, $0.92836$ for matched-random replay, and $0.89602$ for never intervening; the oracle obtains $0.95019$. The registered core-over-accuracy advantage is $0.01171$ ($[0.01131,0.01212]$). Both co-primary lower endpoints exceed zero, the core raises no nuisance alarms, and every registered supervised integrity and premise gate passes (\Cref{fig:mnist-confirmation}).

\begin{figure}[h]
\centering
\includegraphics[width=0.92\linewidth]{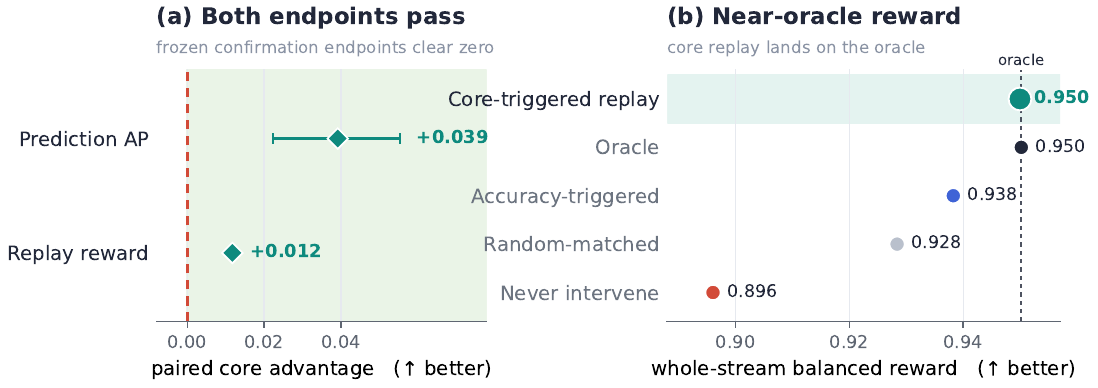}
\caption{\textbf{E12: continual cue-MNIST confirmation (64 complete streams).} \captiona~Registered paired advantages with 95\% intervals from $10{,}000$ complete-stream bootstrap resamples. \captionb~Core-triggered replay reaches near-oracle balanced reward and exceeds the accuracy-triggered, matched-random, and never-intervene controls.}
\label{fig:mnist-confirmation}
\end{figure}

\paragraph{Protocol and reproducibility.}
The compact artifact at \path{code/continual_cue_mnist_v4_confirmation_stats/} contains the original supervised protocol, configuration snapshot, aggregate summary, and 64 per-stream records with SHA-256 provenance. The implementation at \path{code/continual_mnist/}, runnable configurations at \path{code/continual_mnist_configs/}, and accompanying implementation note specify the network and rendering procedure. The note corrects the compact protocol's \emph{three-cue} wording to the third nested set, containing ten cues, and identifies the stored \texttt{clean\_accuracy} field as reference-cue accuracy. The original protocol and numerical records are retained. Source fingerprints match the freeze manifest and its documented numerical addendum; the latter repairs categorical sampling only in the nonprimary policy-gradient arm and leaves the 64 supervised streams unchanged. Figure regeneration uses \path{code/plot_cue_mnist_confirmation.py}; aggregate means were independently recomputed from the packaged records.

\begin{figure}[h]
\centering
\begin{minipage}[c]{0.48\linewidth}
  \centering
  \includegraphics[width=0.78\linewidth]{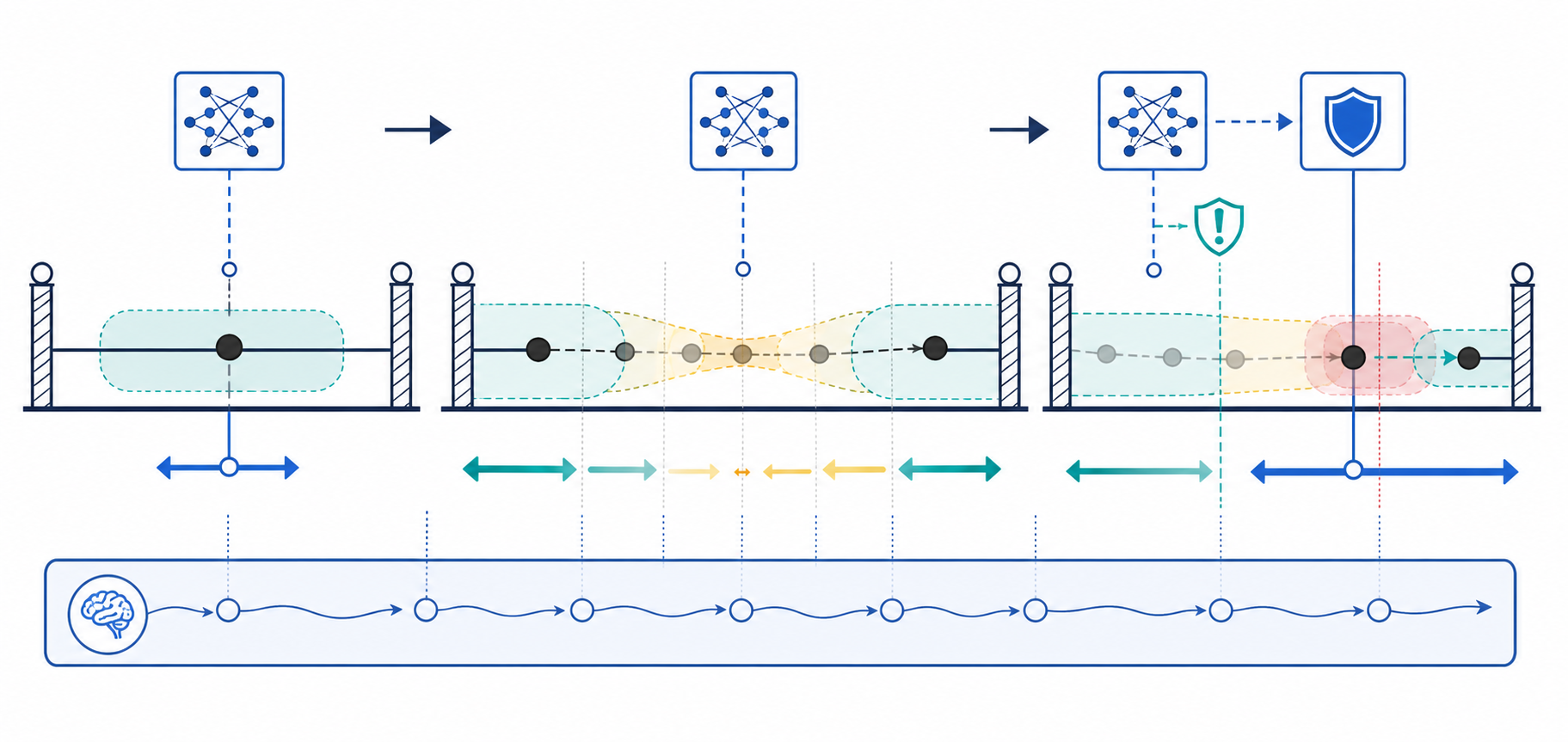}
\end{minipage}\hfill
\begin{minipage}[c]{0.48\linewidth}
  \centering
  \includegraphics[width=0.78\linewidth]{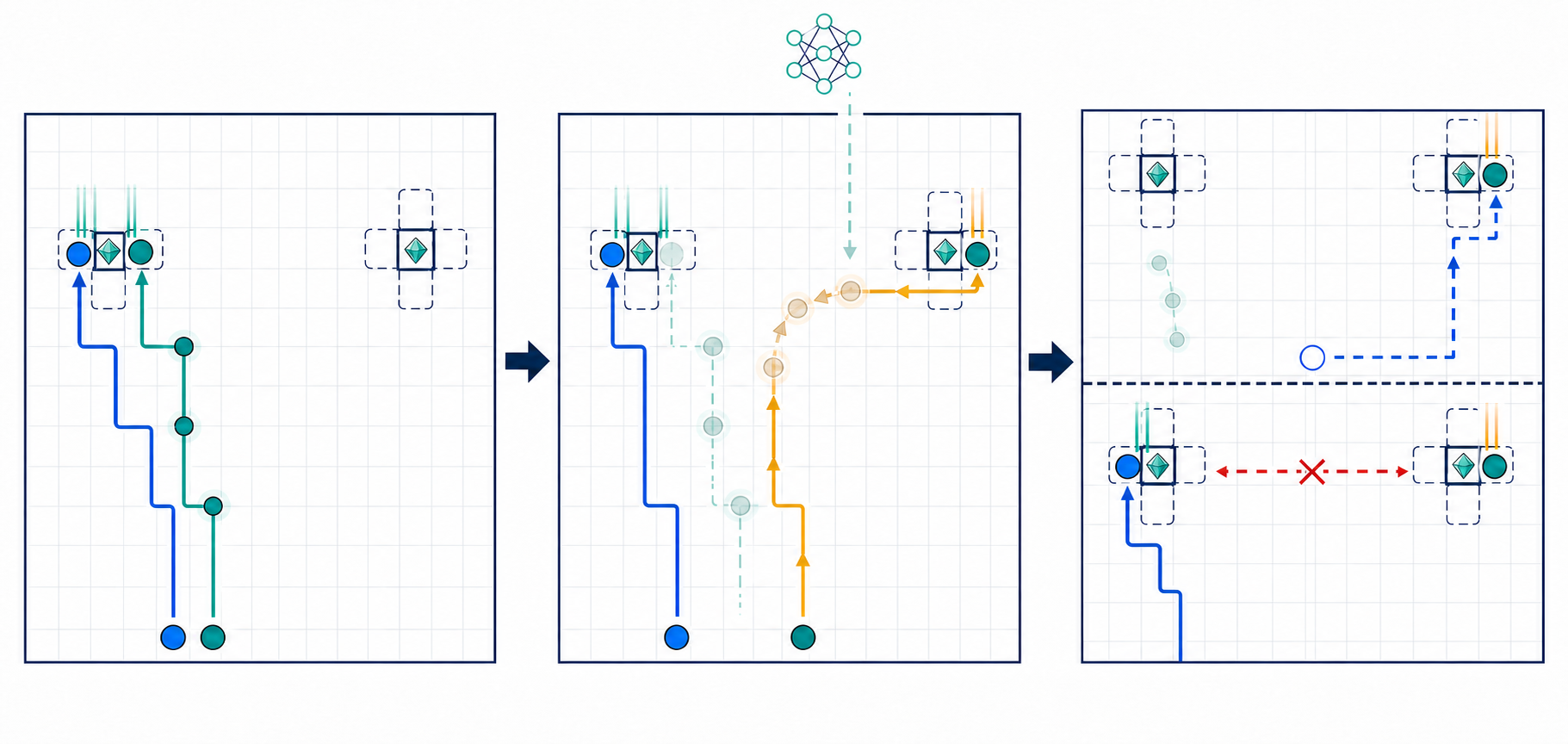}
\end{minipage}
\caption{\textbf{Task schematics.} \captiona~E10 persistent continual control: actuator drift erodes the nominal controller's viability core, which drives intervention. \captionb~E11 LBF coordination fork: peer learning moves behavior between the left and right conventions while the focal policy remains fixed. Illustrations are explanatory; quantitative results appear in \Cref{fig:main-experiments}.}
\label{fig:task-schematics}
\end{figure}

\FloatBarrier
\enlargethispage{2\baselineskip}
\section{LLM Usage Statement}
Language models assisted with developing and editing mathematical proofs by proposing proof outlines and intermediate lemmas, expanding implicit derivations into explicit steps, examining assumptions and edge cases, and providing adversarial review of theorem statements and arguments. They also assisted with copy-editing, \LaTeX{} formatting, figure scripting, terminology lookup, literature discovery across neighboring fields, and non-data task schematics. The author directed, challenged, and revised this assistance, steered the mathematical development, selected and refined the final arguments, and remains responsible for the manuscript's claims, citations, proofs, empirical results, and contents.

\end{document}